\UseRawInputEncoding
\documentclass[twocolumn]{article}
\usepackage{graphicx,subfigure,amssymb,amsthm,appendix,hyperref}
\usepackage[margin=0.8in]{geometry}
\providecommand{\keywords}[1]{\textbf{\textit{Index terms---}} #1}
\newtheorem{theorem}{Theorem}

\begin{document}

\title{String Tightening as a Self-Organizing Phenomenon:\\Computation of Shortest Homotopic Path, Smooth Path, and Convex Hull}
\author{Bonny Banerjee\footnote{This research was supported by participation in the Advanced Decision Architectures Collaborative Technology Alliance sponsored by the U.S. Army Research Laboratory under Cooperative Agreement DAAD19-01-2-0009. This work was done while the author was with the Laboratory for Artificial Intelligence Research, Department of Computer Science \& Engineering, The Ohio State University, Columbus, OH 43210, USA. An earlier version of this article was published as \cite{Banerjee2007}.\newline\indent E-mail: BonnyBanerjee@yahoo.com}, Ph.D.}

\maketitle

\begin{abstract}
The phenomenon of self-organization has been of special interest to the neural network community for decades. In this paper, we study a variant of the Self-Organizing Map (SOM) that models the phenomenon of self-organization of the particles forming a string when the string is tightened from one or both ends. The proposed variant, called the String Tightening Self-Organizing Neural Network (STON), can be used to solve certain practical problems, such as computation of shortest homotopic paths, smoothing paths to avoid sharp turns, and computation of convex hull. These problems are of considerable interest in computational geometry, robotics path planning, AI (diagrammatic reasoning), VLSI routing, and geographical information systems. Given a set of obstacles and a string with two fixed terminal points in a two dimensional space, the STON model continuously tightens the given string until the unique shortest configuration in terms of the Euclidean metric is reached. The STON minimizes the total length of a string on convergence by dynamically creating and selecting feature vectors in a competitive manner. Proof of correctness of this anytime algorithm and experimental results obtained by its deployment are presented in the paper.
\end{abstract}



\section{Introduction}
\label{Introduction}

Self-organization, as a phenomenon, has received
considerable attention from the neural network community in the
last couple of decades. Several attempts have been made to use
neural networks to model different self-organization phenomena.
One of the most well known of such attempts is that of Kohonen's
who proposed the Self-Organizing Map (SOM) \cite{Kohonen2001}
inspired by the way in which various human sensory impressions are
topographically mapped into the neurons of the brain. SOM
possesses the capability to extract features from a
multidimensional data set by creating a vector quantizer by
adjusting weights from common input nodes to M output nodes
arranged in a two dimensional grid. At convergence, the weights
specify the clusters or vector centers of the set of input vectors
such that the point density function of the vector centers tend to
approximate the probability density function of the input vectors.
Several authors in different contexts reported different dynamic
versions of SOM
\cite{Kohonen2001,Kangasetal1990,Fritzke1994,ChoiPark1994,Schweizeretal1991,Obermayeretal1990,FavataWalker1991,RitterKohonen1989,KusumotoTakefuji2006,Fritzke1997}.

In this paper, assuming a string is composed of a sequence of
particles, we claim that the phenomenon undergone by the particles
of the string, when the string is pulled from one or both ends to
tighten it, is that of self-organization, by modeling the
phenomenon using a variant of SOM, called the String Tightening
Self-Organizing Neural Network (STON). We further use the proposed
variant to solve a few well-known practical problems - computation
of shortest path in a given homotopy class, smoothing paths to
avoid sharp turns, and computation of convex hull. Other than
theoretical considerations in computational geometry
\cite{Mitchell2000}, computation of shortest homotopic paths is of
considerable interest in robotics path-planning
\cite{Chosetetal2005}, AI (diagrammatic reasoning)
\cite{Chandra2002,banerjee2003constructing,Chandra2004,Chandra2005,Banerjee2006,BanerjeeChandraPathASC2006,BanerjeePhDthesis2007,chandra2009diagrammatic,BanerjeeChandra2010,BanerjeeChandra2010JAIR,BanerjeeChandra2012}, VLSI routing \cite{Gaoetal1988}, and geographical information systems. Smooth paths are required for
navigation of large robots incapable of taking sharp turns, and
also for handling unexpected obstacles. To generate a path that is
smooth, shorter, collision-free, and is homotopic to the original
path requires generation of the configuration space of a robot
which is computationally expensive and difficult to represent
\cite{Quinlan1993}. Computation of convex hull finds numerous
applications in computational geometry algorithms, pattern
recognition, image processing, and so on. The aim of this paper is
to study the properties of STON and how it might be applied to
solve some practical problems as described above.

The remainder of this paper is organized as follows. In the next section, the STON algorithm is described assuming the given string is sampled at a frequency of at least $d/2$ where $d$ is the minimum distance between the obstacles. Thereafter, an analysis of the algorithm is presented along with proof of its important properties and correctness. Section \ref{Computation of shortest homotopic paths} discusses how STON might be extended when the above constraint on sampling is not met. The extension is used for computation of shortest path in a given homotopy class. Proof of correctness and complexity analysis of the extension are also included. Finally, simulation results are presented from computing shortest homotopic path, smooth path and convex hull using both the original and extended algorithms. The paper concludes with a general discussion.

\section{The STON Algorithm}
\label{The STON Algorithm}

\subsection{Homotopy}
\label{Homotopy}

A string $\pi$ in two dimensional space $(\Re^{2})$ might be
defined as a continuous mapping $\pi:[0,1]\rightarrow \Re^{2}$,
where $\pi(0)$ and $\pi(1)$ are the two terminal points of the
string. A string is simple if it does not intersect itself,
otherwise it is non-simple. Let $\pi_{1}$ and $\pi_{2}$ be two
strings in $\Re^{2}$ sharing the same terminal points i.e.
$\pi_{1}(0)=\pi_{2}(0)$ and $\pi_{1}(1)=\pi_{2}(1)$, and avoiding
a set of obstacles $P\subset \Re^{2}$. The strings $\pi_{1}$,
$\pi_{2}$ are considered to be homotopic to each other or to
belong to the same homotopy class, with respect to the set of
obstacles $P$, if there exists a continuous function
$\Psi:[0,1]\times[0,1]\rightarrow \Re^{2}$ such that \\1.
$\Psi(0,t)=\pi_{1}(t)$ and $\Psi(1,t)=\pi_{2}(t)$, for $0\leq
t\leq 1$
\\2. $\Psi(\lambda,0)=\pi_{1}(0)=\pi_{2}(0)$ and $\Psi(\lambda,1)=\pi_{1}(1)=\pi_{2}(1)$, for $0\leq \lambda \leq 1$ \\3. $\Psi(\lambda,t)\notin P$, for $0\leq \lambda \leq 1$ and $0\leq t \leq 1$. \\Informally, two strings are considered
to be homotopic with respect to a set of obstacles, if they share
the same terminal points and one can be continuously deformed into
the other without crossing any obstacle. Thus homotopy is an
equivalence relation.

Given a string $\pi_{i}$, specified in terms of sampled points,
and a set of obstacles $P$, STON computes a string $\pi_{s}$ such
that $\pi_{i}$ and $\pi_{s}$ belong to the same homotopy class,
and the Euclidean distance covered by $\pi_{s}$ is the shortest
among all strings homotopic to $\pi_{i}$. It is noteworthy that
$\pi_{s}$ is unique and has some canonical form
\cite{GrigorievSlissenko1998}.

\subsection{The Objective}
\label{The Objective}

Assume a string wound around obstacles in $\Re^{2}$ with two fixed
terminal points. A shorter configuration of the string can be
obtained by pulling its terminals. The unique shortest
configuration can be obtained by pulling its terminals until they
cannot be pulled any more. The proposed algorithm models this
phenomenon as a self-organized mapping of the points forming a
given configuration of a string into points forming the desired
shorter configuration of the string. Let us consider a set of $n$
data points or obstacles, $P=\{p_{1},p_{2},...p_{n}\}$,
representing the input signals, and a sequence of variable (say,
$k$) processors $\langle q_{1},q_{2},...q_{k} \rangle$ each of
which (say $q_{i}$) is associated with a weight vector $w_{i}(t)$
at any time $t$. A weight vector represents the position of its
processor in $\Re^{2}$. If the $k$ processors are placed on a
string in $\Re^{2}$, the STON is an anytime algorithm for tuning
the corresponding weights to different domains of the input
signals such that, on convergence, the processors will be located
in such a way that they minimize a distance function, $\phi(w)$,
given by

\begin{equation}
\label{Distance Function} \phi(w)=\displaystyle\sum_{i=1}^{k-1}
\|w_{i+1}(t)-w_{i}(t)\|^{2}
\end{equation}

\noindent where $q_{i}$, $q_{i+1}$ are two consecutive processors
on the string with corresponding weights $w_{i}(t)$, $w_{i+1}(t)$
at any time $t$. The algorithm further guarantees that the final
configuration of the string formed by the sequence of processors
at convergence lies in the same homotopy class as the string
formed by the initial sequence of processors with respect to $P$.
Thus, assuming fixed $q_{1}$ and $q_{k}$, the STON defines the
shortest configuration of the $k$ processors in an unsupervised
manner. The phenomenon undergone by the particles forming the
string is modeled by the processors in the neural network.

\subsection{Initialization of the Network}
\label{Initialization of the Network}

The STON is initialized with a given number of connected
processors, the weight corresponding to each of which is
initialized at a unique point on the given configuration of a
string. A feature vector, presented to the STON, is an attractor
point in $\Re^{2}$ and is either created dynamically or chosen
selectively from the given set of input obstacles $P$. The weight
vectors are updated iteratively on the basis of the created and
chosen feature space $S(t)$;
$S(t)=\{x_{1}(t),x_{2}(t),...x_{k}(t)\}$ being the set of feature
vectors at any time $t$. It is noteworthy that $x_{i}(t)$ is not
necessarily unique. Unlike SOM and many of its variants,
randomized updating of weights does not yield better results in
the STON model; updating weights sequentially or randomly both
yield the same result in terms of output quality as well as
computation time. On convergence, the location of the processors
representing the unique shortest configuration homotopic to the
given configuration of the string is obtained.

\subsection{Creating/Choosing Feature Vectors}
\label{Creating/Choosing Feature Vectors}

A feature vector $x_{i}(t)$ is created if the triangular area
spanned by three consecutive processors $q_{m-1}$, $q_{m}$,
$q_{m+1}$ does not contain any obstacle $p_{j}\in P$. In that
case,

\begin{equation}
\label{Create Feature Vector}
x_{i}(t)=\frac{w_{m-1}(t')+w_{m+1}(t'')}{2}
\end{equation}

\noindent where $t'$, $t''$ assumes the value $t$ if the weight
has not yet been updated in the current iteration or \emph{sweep}
and the value $t+1$ if the weight has been updated in the current
sweep, and $1<m<k$. If an obstacle, say $p_{j}\in P$, lies within
the triangular area spanned by the three consecutive processors
$q_{m-1}$, $q_{m}$, $q_{m+1}$, the feature vector $x_{i}(t)$ is
chosen to be

\begin{equation}
\label{Choose Feature Vector} x_{i}(t)=p_{j}
\end{equation}

For this algorithm, we assume the given string is sampled such
that there cannot exist multiple non-identical obstacles within
the triangular area spanned by any three consecutive processors
(see Appendix \ref{A Finite Sampling Theorem} for how to achieve
such sampling).

\subsection{Updating Weights}
\label{Updating Weights}

The STON evolves by means of a certain processor evolution
mechanism, given by \cite{Rosenblatt1958}

\begin{equation}
\label{Weight Updating Equation}
w_{m}(t+1)=w_{m}(t)+\alpha(t)[x_{j}(t)-w_{m}(t)]
\end{equation}

\noindent where $\alpha(t)$ is the gain term or learning rate
which might vary with time and $0 \leq \alpha(t) \leq1$.
$\alpha(t)$ might be unity only when feature vectors are created
according to eq. \ref{Create Feature Vector}. All the weight
vectors are updated exactly once in a single sweep, indexed by
$t$.

If modification of weights is continued in this process, the processors tend to produce shorter configurations at the end of
each sweep. The weight vectors converge when

\begin{equation}
\label{Convergence Criteria} \|w_{i+1}(t)-w_{i}(t)\| < \epsilon,
\forall i
\end{equation}

\noindent where $\epsilon$ is a predetermined sufficiently small positive scalar quantity.

\section{Analysis of STON}
\label{Analysis of STON}

STON is a variant of SOM. As in SOM, each feature vector in STON
pulls the selected processors in a neighborhood towards itself as
a result of updating in a topologically constrained manner,
ultimately leading to an ordering of processors that minimizes
some residual error. Neighborhood, in SOM, is something like a
smoothing kernel over a two dimensional grid often taken as the
Gaussian which shrinks over time. In STON, all those processors
are included within the neighborhood of a feature vector that form
triangles with their adjacent processors such that the feature
vector lies within their triangles. Such a neighborhood is conceptually similar to that proposed in \cite{BerglundSitte2006} where the neighborhood is not dependent on time but on the nature of the input signals. STON incorporates competitive
learning as the weights adapt themselves to specific chosen
features of the input signal defined in terms of the obstacles.
The residual error is defined in SOM in terms of variance while in
STON, in terms of Euclidean distance. From the inputs the net
adapts itself dynamically in an unsupervised manner to acquire a
stable structure at convergence, thereby manifesting
self-organization \cite{DeWolfHolvoet2005}. The STON possesses
certain key properties which are discussed in this section that
eventually leads to the proof of correctness of the algorithm.

\begin{theorem}\label{Theorem:initial and final sequence of processors are homotopic}
The configuration of a string formed by the sequence of processors at initialization and the same at convergence are homotopic.
\end{theorem}

\begin{proof}
We start by noting that, given a fixed set of obstacles and fixed
terminal points, the homotopy class of a string can be altered
only by crossing any obstacle. The configuration of a string
formed by the sequence of processors at any time $t$ is obtained
by updating the weights with respect to the selected feature
vectors at time $t$-1, the feature vectors being selected
according to eq. \ref{Create Feature Vector} or \ref{Choose
Feature Vector}. In the first case, creation of a feature vector
and updating the weight does not change the homotopy class of the
string as there was no obstacle in the triangular area, hence
updating the weight did not result in crossing any obstacle.

In the second case, a feature vector is selected by eq.
\ref{Choose Feature Vector} from the set of obstacles. A processor
is pulled towards the feature vector by updating its weight. It
requires to be proven that by such selection of feature vectors
and updating of weights, a string cannot cross any obstacle. Let
$q_{i-1}$, $q_{i}$, $q_{i+1}$ be three consecutive processors on a
string with corresponding weights being $w_{i-1}(t')$, $w_{i}(t)$,
$w_{i+1}(t'')$; $t',t'' \in \{t,t+1\}$, and $x_{i}(t)=p_{j}$ be
the selected feature vector (see Fig. \ref{Fig1}). In order to
complete the proof we need to show $p_{j}$ will never be crossed
if -- (i) weight for processor $q_{i}$ is updated, and (ii) weight
for one of the neighbors of $q_{i}$ (i.e. $q_{i-1}$ or $q_{i+1}$)
is updated.

\begin{figure}
\centering
        \includegraphics[width=0.45\textwidth]{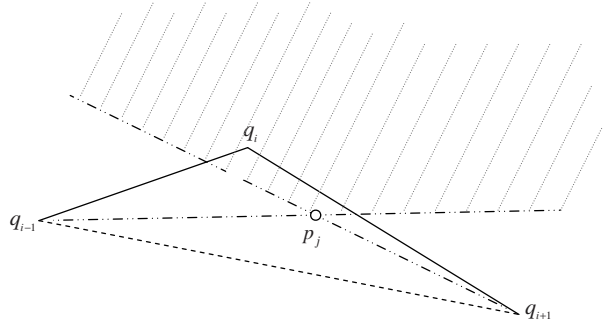}
    \caption{If $p_{j}$ lies inside the triangle formed by $q_{i-1}$, $q_{i}$, $q_{i+1}$, then $q_{i}$ has to lie in the shaded region.}
        \label{Fig1}
\end{figure}

First, note that in order for $p_{j}$ to be inside the triangle
formed by $q_{i-1}$, $q_{i}$, $q_{i+1}$, the processor $q_{i}$ has
to be in the region bounded by extensions of the lines joining
$q_{i-1}$ and $q_{i+1}$ to $p_{j}$ (i.e. the shaded region in Fig.
\ref{Fig1}). Since updating the weight for processor $q_{i}$ pulls
it towards the obstacle $p_{j}$ along a straight line, $q_{i}$
continues to remain within the shaded region after updating,
thereby never letting the segments $\overline{q_{i-1}q_{i}}$ and
$\overline{q_{i}q_{i+1}}$ cross the obstacle $p_{j}$. That proves
condition (i).

\begin{figure}
\centering
        \includegraphics[width=0.45\textwidth]{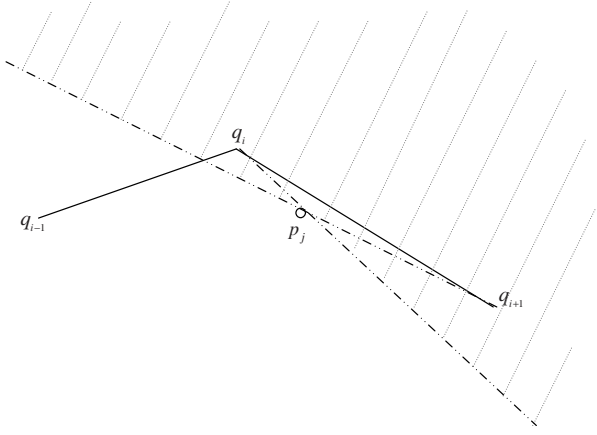}
    \caption{The feature vector for updating the weight for $q_{i+1}$, if not $p_{j}$, can lie only in the shaded region.}
        \label{Fig2}
\end{figure}

Now let us consider the case when the weight for a neighbor of
$q_{i}$, say $q_{i+1}$ without loss of generality, is updated. Let
$q_{i+2}$ be the other neighbor of $q_{i+1}$. Then either the
obstacle $p_{j}$ lies inside the triangle formed by $q_{i}$,
$q_{i+1}$, $q_{i+2}$, or it does not. If $p_{j}$ lies inside, it
is the selected feature vector that pulls $q_{i+1}$ towards itself
and it will never be crossed (due to condition (i)). If $p_{j}$
does not lie inside, then either there lies some other obstacle,
say $p_{j'}$, inside the triangle formed by $q_{i}$, $q_{i+1}$,
$q_{i+2}$, or there does not exist any obstacle inside. In the
former case, $p_{j'}$ is the selected feature vector while in the
later, the feature vector is created according to eq. \ref{Create
Feature Vector}. In any case, the feature vector can lie only in
the partition, bounded by the extensions of the lines joining
$q_{i}$ and $q_{i+1}$ to $p_{j}$, in which $q_{i+1}$ lies (shaded
region in Fig. \ref{Fig2}). Thus, updating the weight for
$q_{i+1}$ will not make the string cross the obstacle $p_{j}$. In
general, updating the neighbors of $q_{i}$ will not make the
string cross $p_{j}$, proving condition (ii).

From this we conclude that updating a weight vector with respect
to a created or selected feature vector does not change the
homotopy class of a string. Hence, the configurations of a string
at the end of consecutive sweeps are homotopic. But homotopy is a
transitive relation. This concludes the proof that the
configurations of a string at initialization and at convergence
are homotopic.
\end{proof}

\begin{theorem}\label{Theorem:length of string decreases with each iteration}
The Euclidean distance covered by the configuration of a string formed by the sequence of processors at time $t+1$ is less than the same at time $t$.
\end{theorem}

\begin{proof}
Let us consider a triangle formed by three consecutive processors
$q_{i-1}$, $q_{i}$, $q_{i+1}$ with corresponding weights
$w_{i-1}(t')$, $w_{i}(t)$, $w_{i+1}(t'')$; $t',t''\in \{t,t+1\}$,
on a configuration of a string (see Fig. \ref{Fig3}). Let
$w_{i}(t+1)$ be the weight vector after updating $w_{i}(t)$ with
respect to a feature vector either created according to eq.
\ref{Create Feature Vector} or chosen according to eq. \ref{Choose
Feature Vector}. We are required to prove that the Euclidean
distance covered by a configuration of the string from
$q_{i-1}[w_{i-1}(t')]$ to $q_{i+1}[w_{i+1}(t'')]$ via
$q_{i}[w_{i}(t+1)]$ is less than the same via $q_{i}[w_{i}(t)]$.

In order to prove that, we extend the line segment
$\overline{q_{i-1}[w_{i-1}(t')]q_{i}[w_{i}(t+1)]}$ to intersect
the line segment $\overline{q_{i}[w_{i}(t)]q_{i+1}[w_{i+1}(t'')]}$
at a point, say $A$ located at $a$. Then, from Fig. \ref{Fig3},
using the Triangle Inequality, we get
\\

\begin{math}
\|w_{i-1}(t')-w_{i}(t+1)\| + \|w_{i}(t+1)-a\| < \|w_{i-1}(t')-w_{i}(t)\| \\ + \|w_{i}(t)-a\|
\end{math}
\\

\begin{math}
\|w_{i}(t+1)-w_{i+1}(t'')\| < \|w_{i}(t+1)-a\| + \|a-w_{i+1}(t'')\|
\end{math}
\\

\noindent From the above inequalities, we get
\\

\begin{math}
\|w_{i-1}(t')-w_{i}(t+1)\| + \|w_{i}(t+1)-w_{i+1}(t'')\| <\\
\|w_{i-1}(t')-w_{i}(t)\| + \|w_{i}(t)-a\| + \|a-w_{i+1}(t'')\|
\end{math}
\\

\begin{figure}
\centering
        \includegraphics[width=0.47\textwidth]{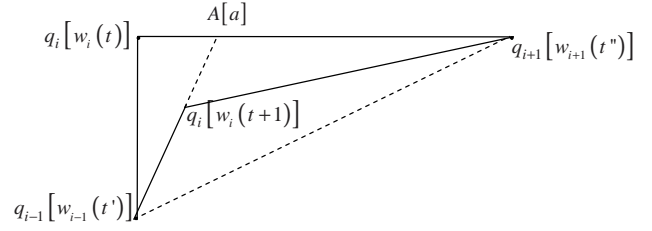}
    \caption{The Euclidean distance covered by a configuration of the
string from $q_{i-1}[w_{i-1}(t')]$ to $q_{i+1}[w_{i+1}(t'')]$ via
$q_{i}[w_{i}(t+1)]$ is less than the same via $q_{i}[w_{i}(t)]$.}
        \label{Fig3}
\end{figure}

\noindent Thus, at any time $t$, after updating a weight $w_{i}(t)$, we have
\\

\begin{math}
\|w_{i-1}(t')-w_{i}(t+1)\| + \|w_{i}(t+1)-w_{i+1}(t'')\| <\\ \|w_{i-1}(t')-w_{i}(t)\| + \|w_{i}(t)-w_{i+1}(t'')\|
\end{math}
\\

\noindent i.e. updating $w_{i}(t)$ contributes to minimization of
the sum of lengths of the segments
$\overline{q_{i-1}[w_{i-1}(t')]q_{i}[w_{i}(t)]}$ and
$\overline{q_{i}[w_{i}(t)]q_{i+1}[w_{i+1}(t'')]}$. Each weight is
updated exactly once in every sweep. Thus updating a weight
contributes to minimization of length of the current configuration
of the string at every sweep, thereby minimizing $\phi(w)$.
\end{proof}

Theorem \ref{Theorem:initial and final sequence of processors are homotopic} shows the STON algorithm guarantees that the final
configuration $\pi_{s}$ of the string belongs to the same homotopy class as its initial configuration $\pi_{i}$. By Theorem \ref{Theorem:length of string decreases with each iteration}, it can be seen that if sufficient number of sweeps are computed, the shortest configuration in the given homotopy class can be reached. Thus the proposed algorithm is correct with respect to the goal of obtaining the shortest homotopic configuration of a string as defined in section \ref{The Objective}. This, however, does not guarantee that the optimum solution will always be reached. It is possible for the algorithm to get stuck at suboptimal solutions, a situation that can be averted by choosing $\alpha$ much less than unity when feature vectors are selected by eq. \ref{Choose Feature Vector}. We shall discuss this issue further in section \ref{Experimental Results}.

\section{Computation of shortest homotopic paths}
\label{Computation of shortest homotopic paths}

The solution to the problem of computing the shortest homotopic
path can be viewed as an instance of pulling a string to tighten
it, where the given path corresponds to the initial configuration
of the string while the shortest homotopic path corresponds to the
tightened configuration of the same string. The STON assumes a
string to be sampled such that there exists at most one obstacle
in the triangle formed by any three consecutive processors. In
this section, given a path $\pi_{i}$, we propose an extension of
STON to do away with that assumption and apply the extension for
computing the shortest homotopic path $\pi_{s}$ with respect to a
given set of obstacles $P$, where $\pi_{i}$ might be simple or
non-simple. The set of obstacles $P$ is specified as a set of
points in $\Re^{2}$ with no assumption being made about their
connectivity. The input path $\pi_{i}$ is specified in terms of a
pair of terminal points and either a mathematical equation or a
sequence of points in $\Re^{2}$. In the former case, the path is
sampled to obtain the sequence of points in $\Re^{2}$.

\subsection{Extending STON for sparsely sampled paths}
\label{Extending STON for sparsely sampled paths}

When a path is sampled sparsely, it can no longer be guaranteed
that at the time of choosing a feature vector, there will exist
only one obstacle within the triangular area spanned by any three
consecutive processors. Hence, the algorithm might fail to perform
correctly as Theorem \ref{Theorem:initial and final sequence of processors are homotopic} no longer holds true (see Fig. \ref{Fig4}).

\begin{figure}
\centering
        \includegraphics[width=0.45\textwidth]{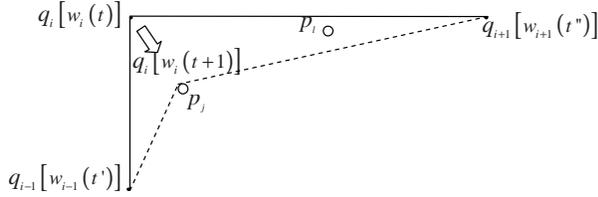}
    \caption{The STON algorithm might fail to perform correctly if
the given path is sparsely sampled, as illustrated in the example
above. The configuration formed by $q_{i-1}[w_{i-1}(t')]$,
$q_{i}[w_{i}(t)]$, $q_{i+1}[w_{i+1}(t'')]$ and that formed by
$q_{i-1}[w_{i-1}(t')]$, $q_{i}[w_{i}(t+1)]$,
$q_{i+1}[w_{i+1}(t'')]$ are not homotopic to each other as the
obstacle $p_{l}$ has been crossed.}
        \label{Fig4}
\end{figure}

Let us assume the given path is very sparsely sampled. In that
case, whenever more than one obstacle point is encountered within
a triangle formed by three consecutive processors $q_{i-1}$,
$q_{i}$, $q_{i+1}$, the centroid, say C, of the obstacle points
lying within the triangle is computed. The line segments joining
the processors $q_{i-1}$ and $q_{i+1}$ to C partitions the
obstacle space within the triangle into two disjoint parts. The
convex hull of the obstacle points lying within the partition
adjacent to the processor $q_{i}$ is computed (see Fig.
\ref{Fig5a}). A new processor is introduced and initialized near
each vertex of the convex hull\footnote{See Appendix
\ref{Introduce processors in convex hull} for details of how we
introduce processors near each vertex of the convex hull.} (see
Fig. \ref{Fig5b}). The indices of all processors and their
corresponding weights are updated. The processor $q_{i}$ is
considered as "useless" and is deleted. A similar notion of
useless units has been used in \cite{Fritzke1997}.

\begin{figure}
\centering
    \subfigure[The convex hull of the obstacle
points lying in the partition, formed by the line segments joining
the processors $q_{i-1}$ and $q_{i+1}$ to C, adjacent to $q_{i}$,
is shown.]
    {
        \label{Fig5a}
        \includegraphics[width=0.4\textwidth]{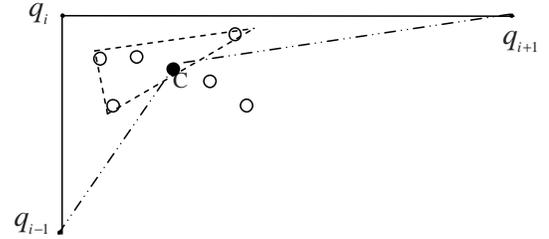}
    }
    \vfill
    \subfigure[Introduction of new processors $q_{i}$, $q_{i+1}$, $q_{i+2}$ near the
vertices of the convex hull.]
    {
        \label{Fig5b}
        \includegraphics[width=0.4\textwidth]{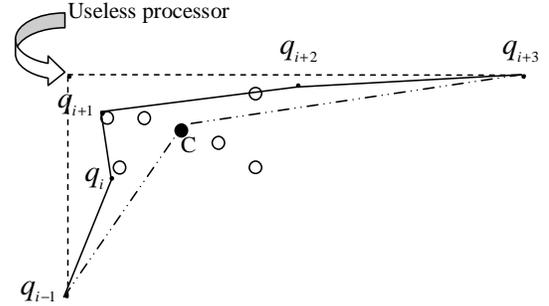}
    }
    \caption{C is the centroid of all obstacle points lying within
the triangle formed by $q_{i-1}$, $q_{i}$, $q_{i+1}$.}
\label{Fig5} 
\end{figure}

We claim, this extension of STON works correctly in all cases. Let
us, for a contradiction, assume that there exists an obstacle in
the partition not adjacent to the processor $q_{i}$ (see Fig.
\ref{Fig5a}) at which a processor has to be introduced in order to
obtain the shortest homotopic path. That is, the shortest
homotopic path will pass through an obstacle in the partition not
adjacent to the processor $q_{i}$. Let $p_{j}$ be such an obstacle
point (see Fig. \ref{Fig6}). The line from $q_{i-1}$ through
$p_{j}$ partitions the triangle formed by $q_{i-1}$, $q_{i}$,
$q_{i+1}$ into two disjoint parts. If the shortest homotopic path
passes through $p_{j}$, there cannot exist any obstacle point in
the partition, formed by the line from $q_{i-1}$ through $p_{j}$,
adjacent to $q_{i}$. But in that case, the centroid C cannot lie
in the partition, formed by the line from $q_{i-1}$ through
$p_{j}$, adjacent to $q_{i}$. Hence, a contradiction. If the
shortest homotopic path passes through $p_{j}$, the centroid C
will lie in the partition, formed by the line from $q_{i-1}$
through $p_{j}$, not adjacent to $q_{i}$. In that case, the
obstacle point $p_{j}$ lies in the partition, formed by line
segments joining processors $q_{i-1}$ and $q_{i+1}$ to C, adjacent
to $q_{i}$. Hence the claim follows.

\begin{figure}
\centering
        \includegraphics[width=0.42\textwidth]{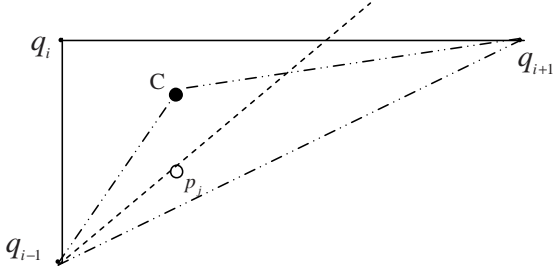}
    \caption{The shortest homotopic path cannot pass through any
obstacle point lying in the partition, formed by line segments
joining processors $q_{i-1}$ and $q_{i+1}$ to C, not adjacent to
$q_{i}$.}
        \label{Fig6}
\end{figure}

\subsection{The algorithm: Extension of STON}
\label{The algorithm: Extension of STON}

\noindent \textbf{Input:} Set of obstacles, Initial string configuration (i.e. initial sequence of processors)\\
\textbf{Output:} Final string configuration (i.e. final sequence of processors)\\
1. initialize the weights\\
2. $t\leftarrow 0$\\
3. while convergence criteria (eq. \ref{Convergence Criteria}) not satisfied, do\\
4. \hspace{3mm} $t\leftarrow t+1$\\
5. \hspace{3mm} for each processor on path, do\\
6. \hspace{6mm} $z\leftarrow$number of obstacles inside triangle formed with neighboring processors\\
7. \hspace{6mm} if $z=0$,\\
8. \hspace{9mm} create feature vector (eq. \ref{Create Feature Vector})\\
9. \hspace{9mm} update weight (eq. \ref{Weight Updating Equation})\\
10. \hspace{4mm} if $z=1$,\\
11. \hspace{7mm} select feature vector (eq. \ref{Choose Feature Vector})\\
12. \hspace{7mm} update weight (eq. \ref{Weight Updating Equation})\\
13. \hspace{4mm} if $z>1$,\\
14. \hspace{7mm} compute convex hull of the selected obstacles\\
15. \hspace{7mm} introduce new processors (Appendix \ref{Introduce processors in convex hull}) and update their weights (eq. \ref{Weight Updating Equation})\\
16. return sequence of processors\\

First, note that the algorithm for the original version of STON
comprised of steps $1$ through $12$ of the above algorithm.
Computational complexity of step $6$ is $O(logn+m)$ where $n$ is
the number of obstacle points and $m$ is the number of obstacle
points inside the triangle formed by a processor and its adjacent
neighbors, $0\leq m\leq n$. This complexity can be achieved by a
one-time construction of a $2D$ range tree of the obstacle points
in $O(nlogn)$ time. Querying the tree requires $O(logn+m)$ time
using fractional cascading \cite{deBergetal1997}. On average,
$m=\frac{n}{k}$ where $k$ is the number of processors. Thus,
complexity of STON is $O(logn+n/k)$ per processor per sweep,
assuming the input path has been sampled at half the minimum
distance between the obstacles. The purpose of extending STON is
to eliminate the constraint on sampling. As a result, steps
$13-15$ had to be introduced which uses the algorithm recursively
for computing convex hull. Let $T(n)$ be the complexity of
extension of STON for each sweep and $k$ be the number of
processors at the end of a sweep. Then, from the above algorithm,
we get

\begin{equation}
\label{Complexity}
T(n) = k(r_{c}T(m) + logn + m)
\end{equation}

\noindent where $r_{c}$ is the number of sweeps required to
compute convex hull. The convex hull is computed to determine the
number and locations of new processors that have to be introduced
so that there does not exist more than one obstacle in any
triangle formed by three consecutive processors. For this purpose,
it is sufficient to compute just one sweep of the convex hull
instead of a tight convex hull. This strategy saves computational
costs as the newly added processors will eventually not remain on
the convex hull of the obstacles but will remain on the path.
Therefore,

\begin{equation}
\label{Simplified complexity}
T(n) = k(T(\frac{n}{k}) + logn + \frac{n}{k}) = O(n(logk+log_{k}n))
\end{equation}

Thus the complexity of extension of STON is
$O(\frac{n}{k}(logk+log_{k}n))$ per processor per sweep. As the
number of processors ($k$) increases, $m\rightarrow 1$, and the
complexity of extension of STON becomes $O(logn)$. Thus the
extension of STON starts with a complexity of
$O(\frac{n}{k}(logk+log_{k}n))$ and reaches a complexity of
$O(logn)$ when no more processors are required to be introduced. It is interesting to note that the complexity of STON is comparable with that of some of the recently proposed variants of SOM \cite{KusumotoTakefuji2006,PalDattaPal2001}.

In computational geometry, many researchers have proposed
algorithms to solve this problem with the primary goal of
minimizing computational complexity (refer to \cite{Mitchell2000}
for a detailed review). Efrat et al \cite{Efrat2002} and
Bespamyatnikh \cite{Bespamyatnikh2003a} have independently
proposed output sensitive algorithms for the problem. Their
algorithms tackle the problem for simple and non-simple paths in
different ways, with the one for non-simple paths having higher
complexity. Bespamyatnikh's algorithm for non-simple paths
achieves $O(log^{2}n)$ time per output vertex. If the terminal points of a given path are not fixed, the
resulting problem is NP-hard \cite{Richards1984} which has not
been dealt with in this paper.

\section{Experimental Results}
\label{Experimental Results}

In this section, we present experimental results obtained by
deploying STON to different data sets for different purposes. The
extension of STON has been used to compute shortest homotopic
paths, smooth paths, and convex hulls. In Fig. \ref{Fig7}, the
performance of STON is illustrated assuming $\alpha$ is assigned a
large value close to unity when feature vectors are chosen
according to eq. \ref{Choose Feature Vector}. In that case, the
algorithm might fail to perform optimally as the final path might
cling to undesired obstacles, as shown by the arrow in Fig.
\ref{Fig7}. This happens because once a processor falls on an
obstacle, which might happen for some processors before
convergence if $\alpha$ is large, the processor fails to let the
obstacle go as the obstacle continues to remain within its
triangle and the processor has no memory of which direction it
proceeded from. Such performance from STON can only be averted by
choosing $\alpha$ much smaller than unity when feature vectors are
selected by eq. \ref{Choose Feature Vector}. This provides ample
time for the processors to distribute themselves along the path
before coming close to any obstacle. For our experiments, $\alpha$
was chosen as follows

\begin{equation}
\label{alpha}
\alpha(t)=
  \left\{
  \begin{array}{ll}
    1.0 & \mbox{if feature vector is created}\\
    \beta(1+\frac{t}{T}) & \mbox{if feature vector is selected}
  \end{array}
  \right.
\end{equation}

\noindent where $\beta$ is the learning constant, $0<\beta<0.5$,
and $T$ is the total number of sweeps that STON is expected to
converge within. Typically, $\beta$ and $T$ are assigned values
0.01 and 5000 respectively. Thus, initially a processor proceeds
slowly towards the chosen obstacle but the rate of proceeding
increases as more and more sweeps are computed. This prevents STON
from converging at suboptimal solutions. Throughout our
experiments, $\epsilon$ is chosen to be 0.001\% of the maximum
distance covered along any one dimension by the obstacles.

\begin{figure}
\centering
        \includegraphics[width=0.3\textwidth,trim=0 0 0 0,clip]{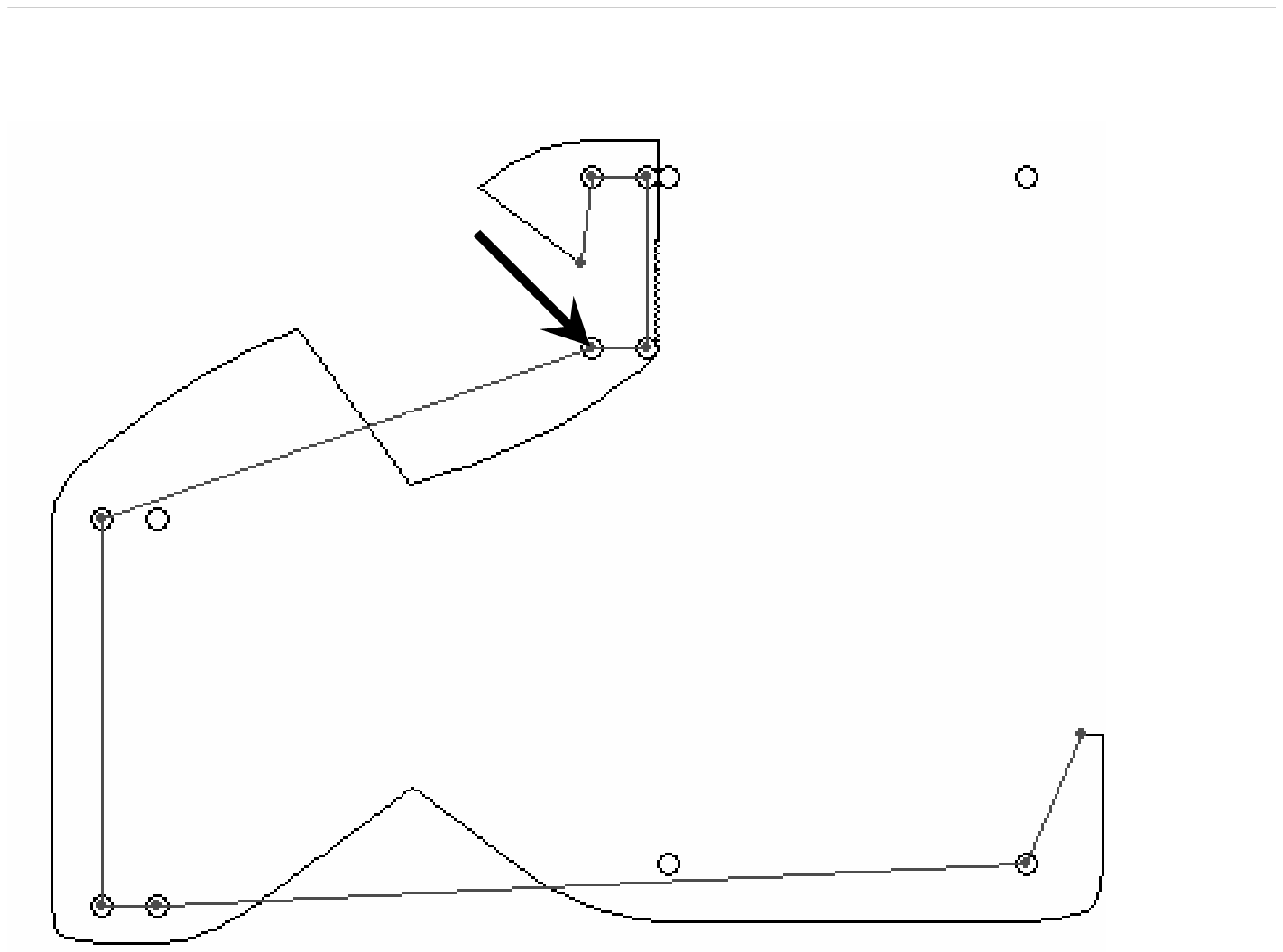}
    \caption{STON works incorrectly for large values of $\alpha$. Circles
represent point obstacles, while dark and light lines represent
initial and tightened configurations of a path respectively.}
        \label{Fig7}
\end{figure}

\begin{figure}[t!]
    \subfigure[At initialization.]
    {
        \label{Fig8a}
        \includegraphics[width=0.22\textwidth]{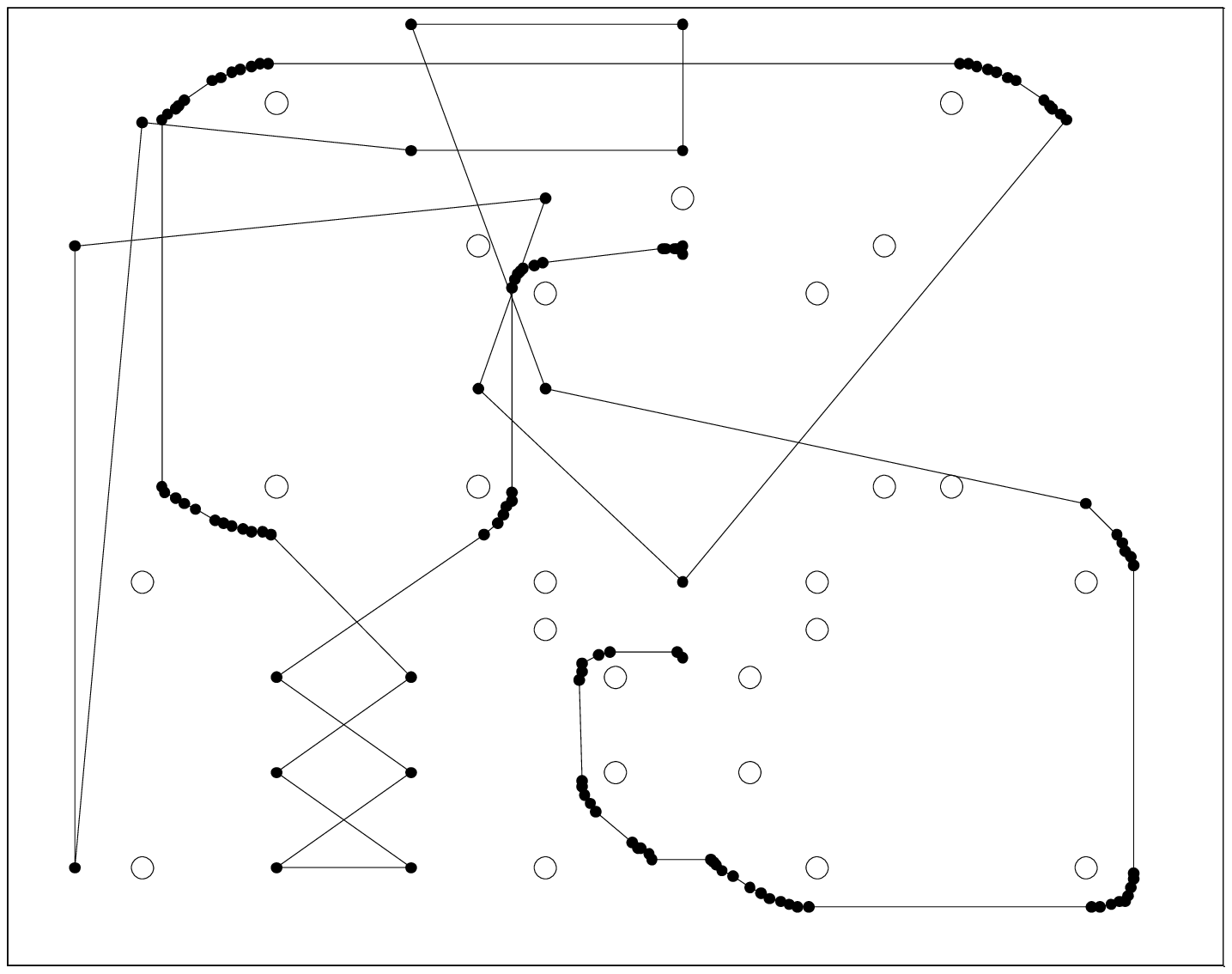}
    }
    \hfil
    \subfigure[After first sweep.]
    {
        \label{Fig8b}
        \includegraphics[width=0.22\textwidth]{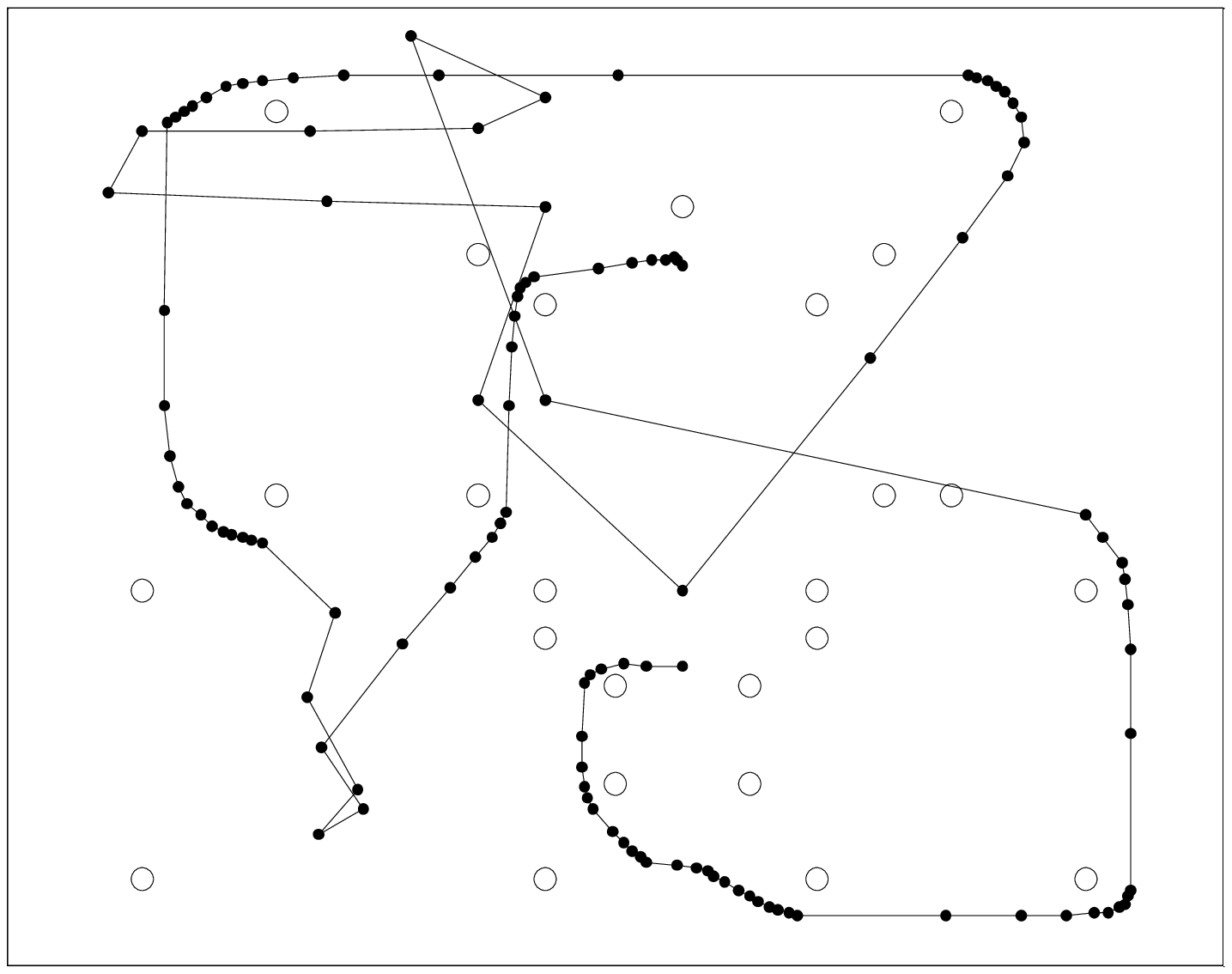}
    }
    \vfill
    \subfigure[After 5 sweeps.]
    {
        \label{Fig8c}
        \includegraphics[width=0.22\textwidth]{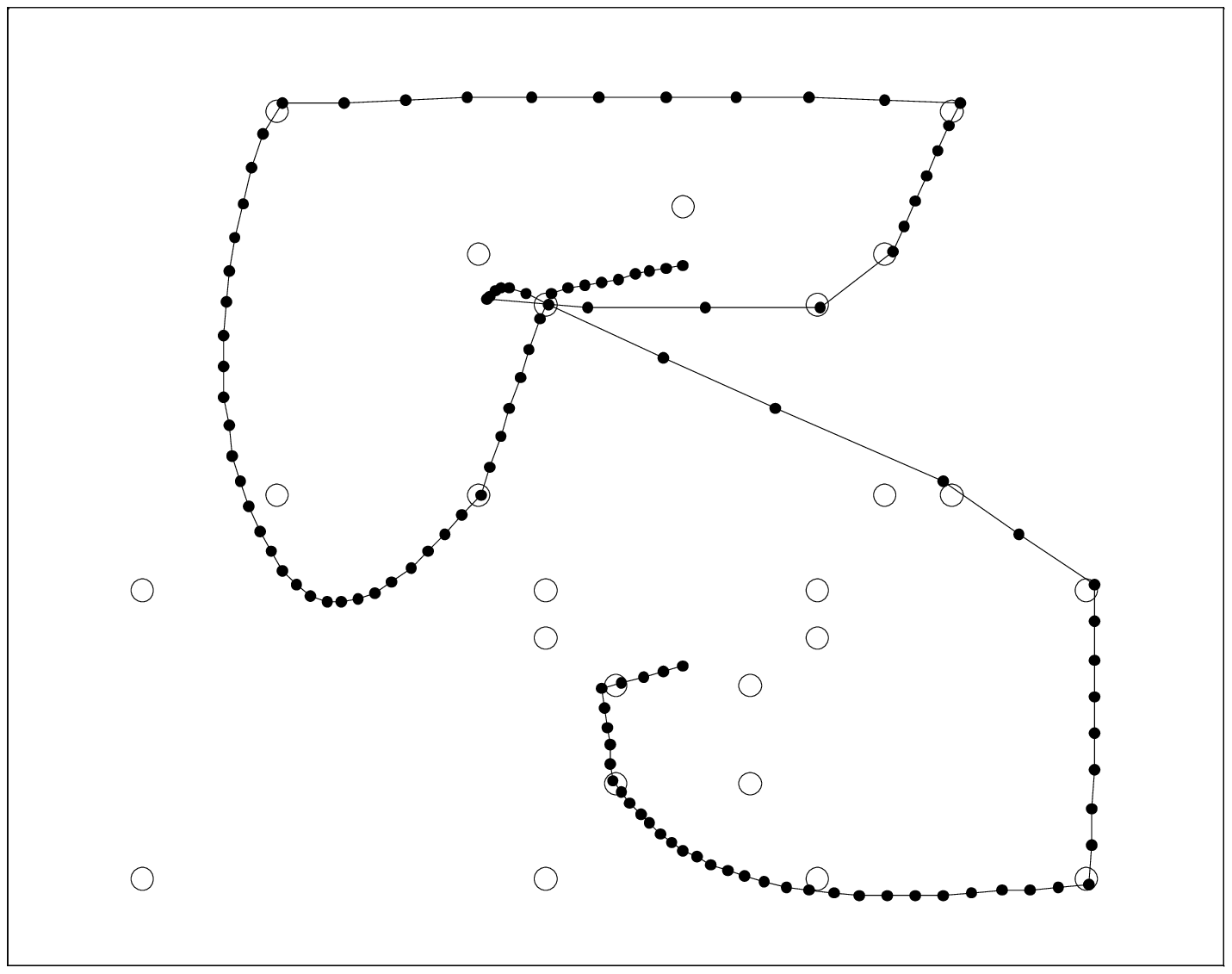}
    }
    \hfil
    \subfigure[After 10 sweeps.]
    {
        \label{Fig8d}
        \includegraphics[width=0.22\textwidth]{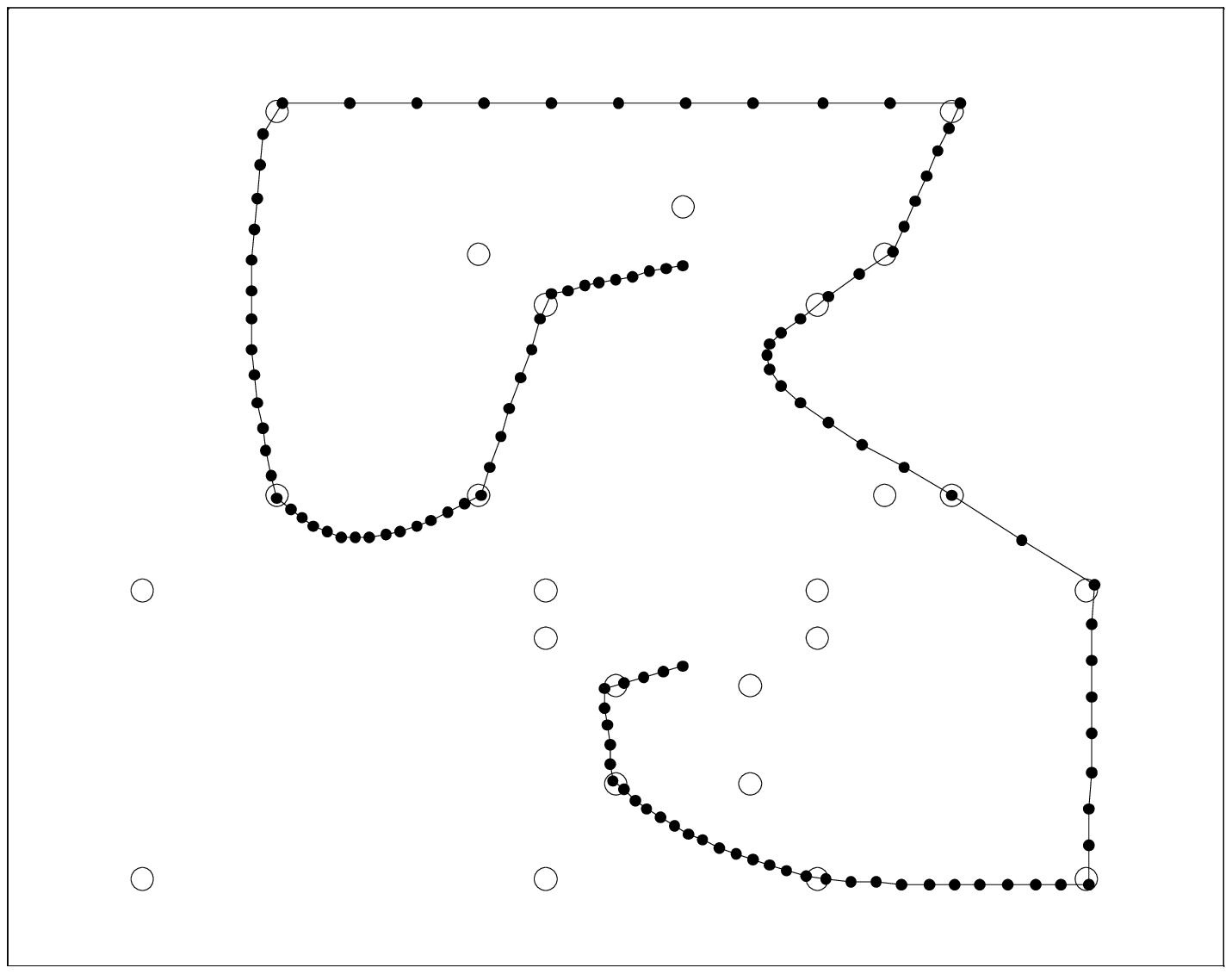}
    }
    \vfill
    \subfigure[After 15 sweeps.]
    {
        \label{Fig8e}
        \includegraphics[width=0.22\textwidth]{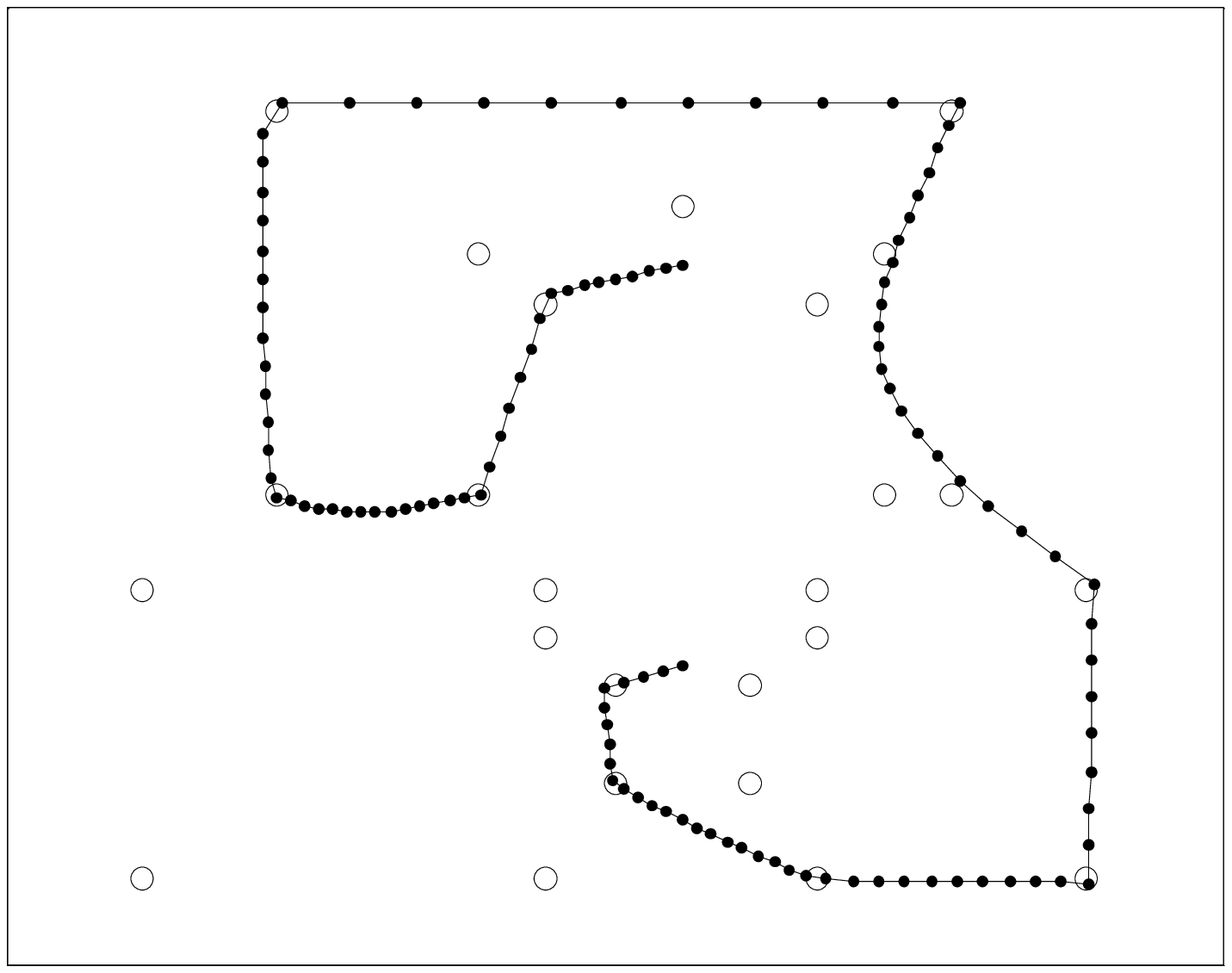}
    }
    \hfil
    \subfigure[After 20 sweeps.]
    {
        \label{Fig8f}
        \includegraphics[width=0.22\textwidth]{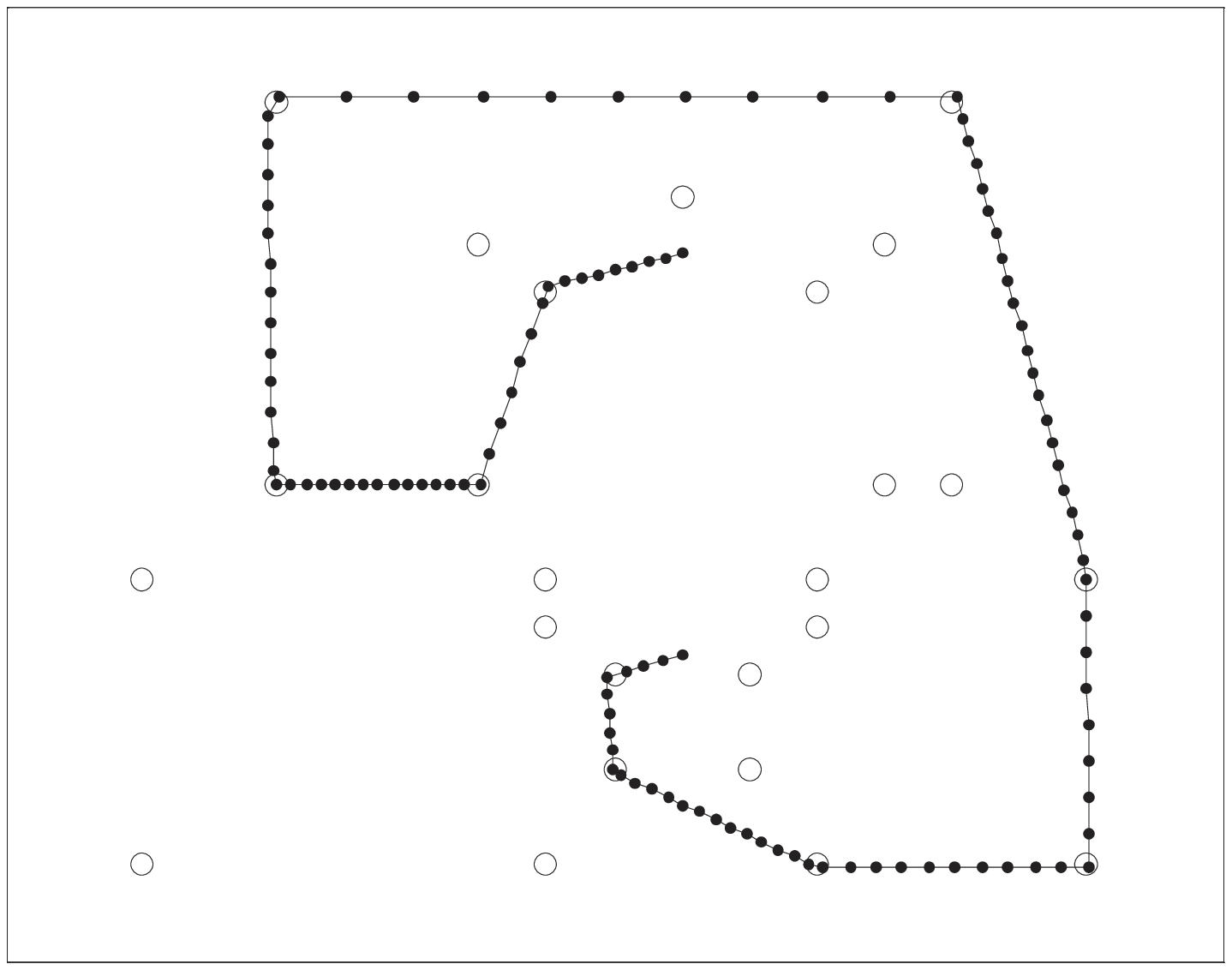}
    }
    \caption{STON applied to a non-simple path that turns out to be
simple when tightened. Circles denote point obstacles while dots
denote locations of processors on the path.}
\label{Fig8} 
\end{figure}

\begin{figure}[h!]
    \subfigure[At initialization.]
    {
        \label{Fig9a}
        \includegraphics[width=0.22\textwidth]{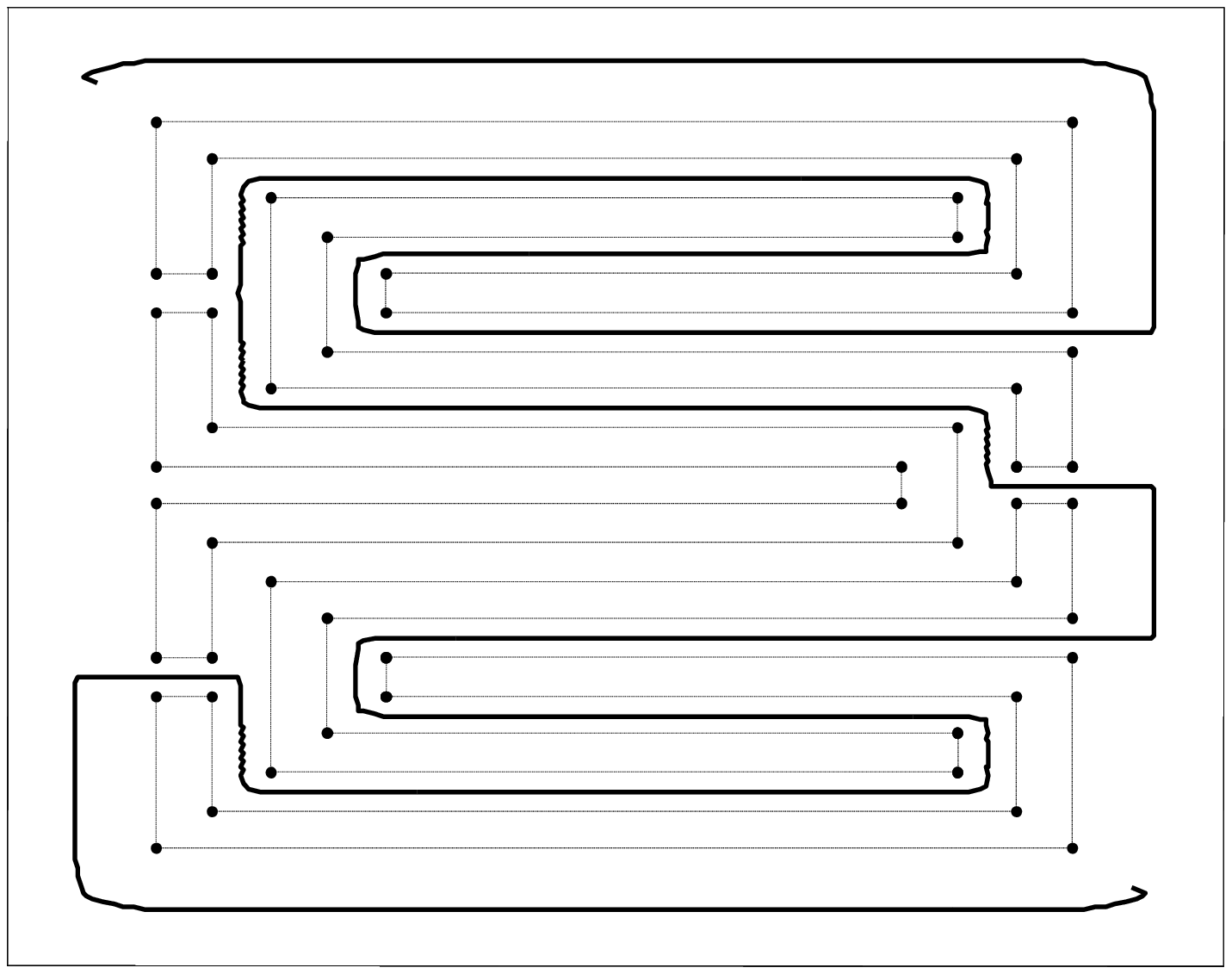}
    }
    \hfil
    \subfigure[After 15 sweeps.]
    {
        \label{Fig9b}
        \includegraphics[width=0.22\textwidth]{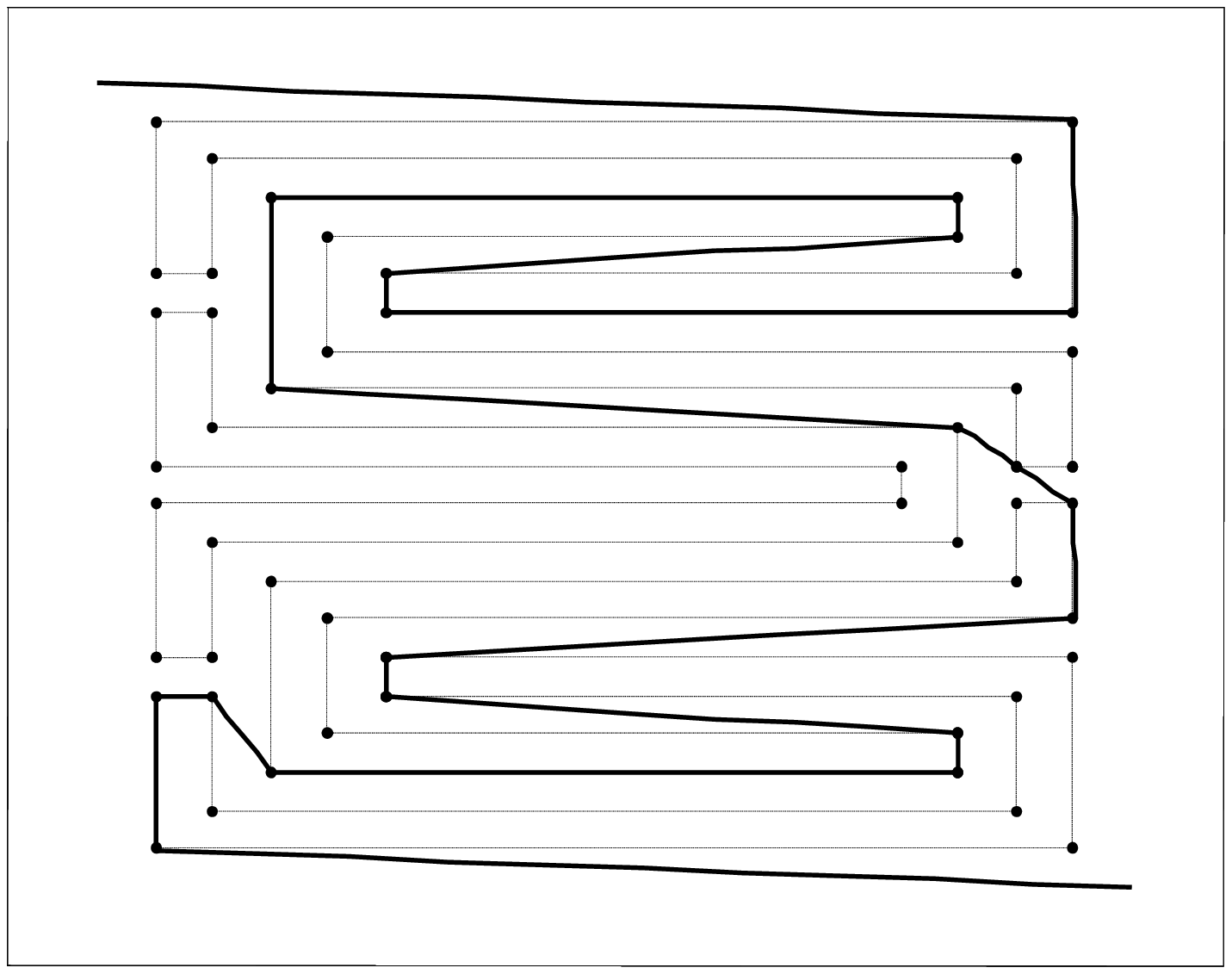}
    }
    \caption{Performance of STON in a structured obstacle
environment. Dotted lines represent the contours of the obstacles
while firm lines represent paths.}
\label{Fig9} 
\end{figure}

Fig. \ref{Fig8a} illustrates a complicated configuration of a path
that has not been sampled uniformly. STON was applied to shorten
this path and the configurations reached after 1, 5, 10, 15, 20
sweeps are shown in Fig. \ref{Fig8}. It can be seen that the
processors gradually distribute themselves evenly along the path
due to their weights being updated with respect to created feature
vectors (eq. \ref{Create Feature Vector}). This and the assignment
of $\alpha$ according to eq. \ref{alpha} help STON to avoid
suboptimal convergence which could have been the case after 10
sweeps (see Fig. \ref{Fig8d}). The correct shortest configuration
was reached within 20 sweeps.

Fig. \ref{Fig9} illustrates the performance of STON in a
structured environment where obstacles are not point objects but
are two dimensional shapes. The algorithm converged within 15
sweeps. In a structured environment, STON considers the points
forming the two dimensional shapes as point obstacles and does not
use their connectivity information. The algorithm performs equally
well in structured as well as unstructured environments.

\begin{figure}
    \subfigure[At initialization.]
    {
        \label{Fig10a}
        \includegraphics[width=0.22\textwidth]{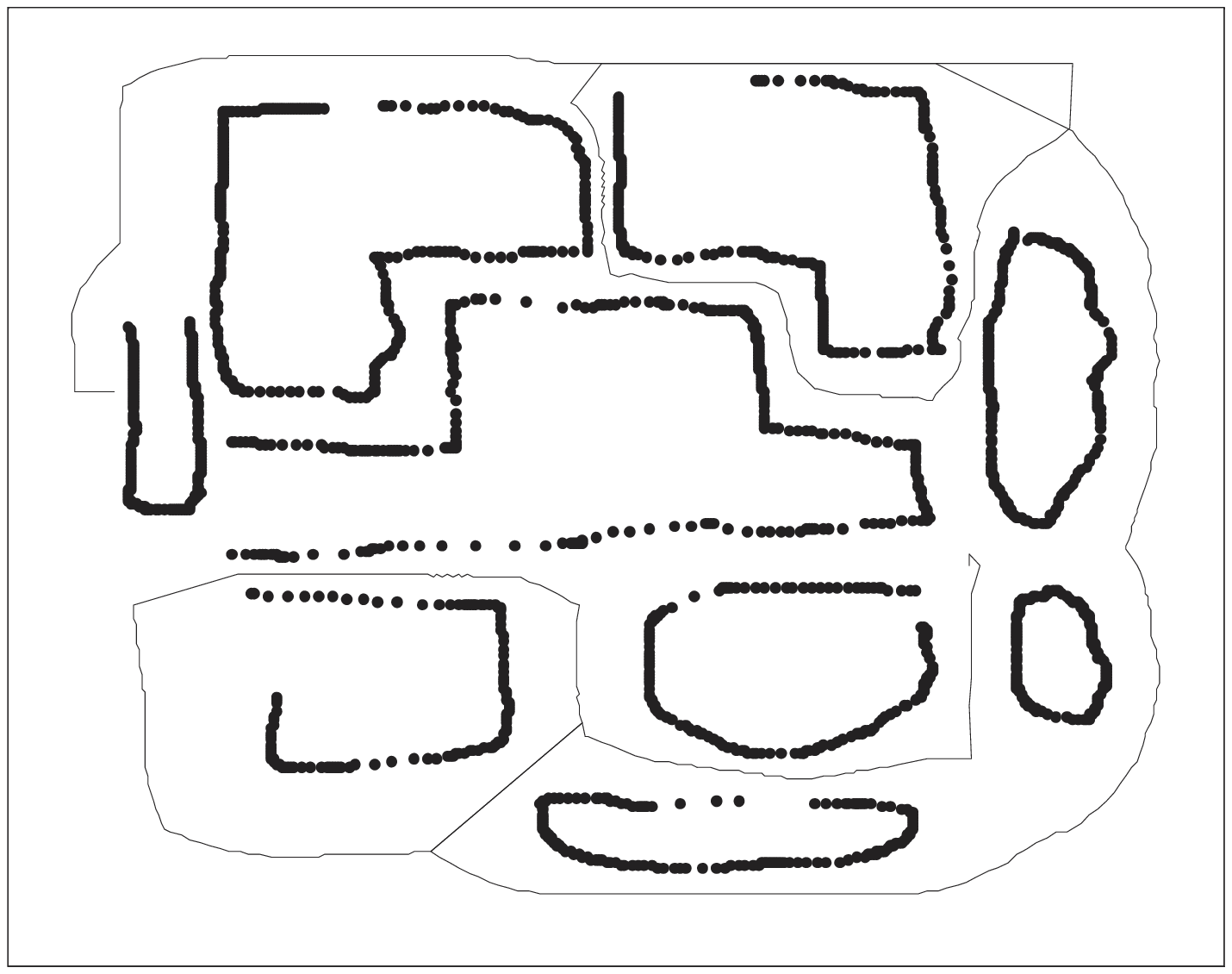}
    }
    \hfil
    \subfigure[After first sweep.]
    {
        \label{Fig10b}
        \includegraphics[width=0.22\textwidth]{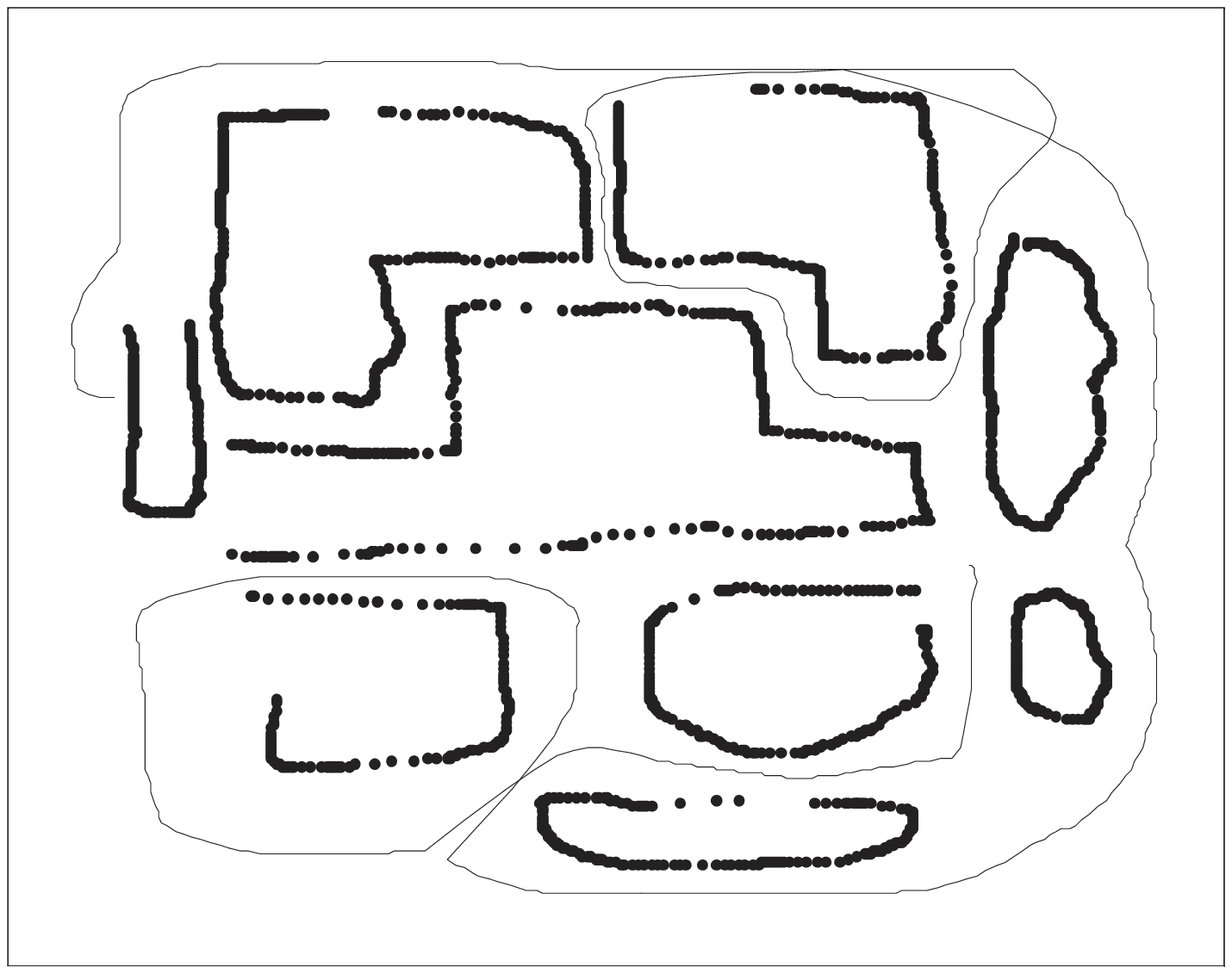}
    }
    \vfill
    \subfigure[After 10 sweeps.]
    {
        \label{Fig10c}
        \includegraphics[width=0.22\textwidth]{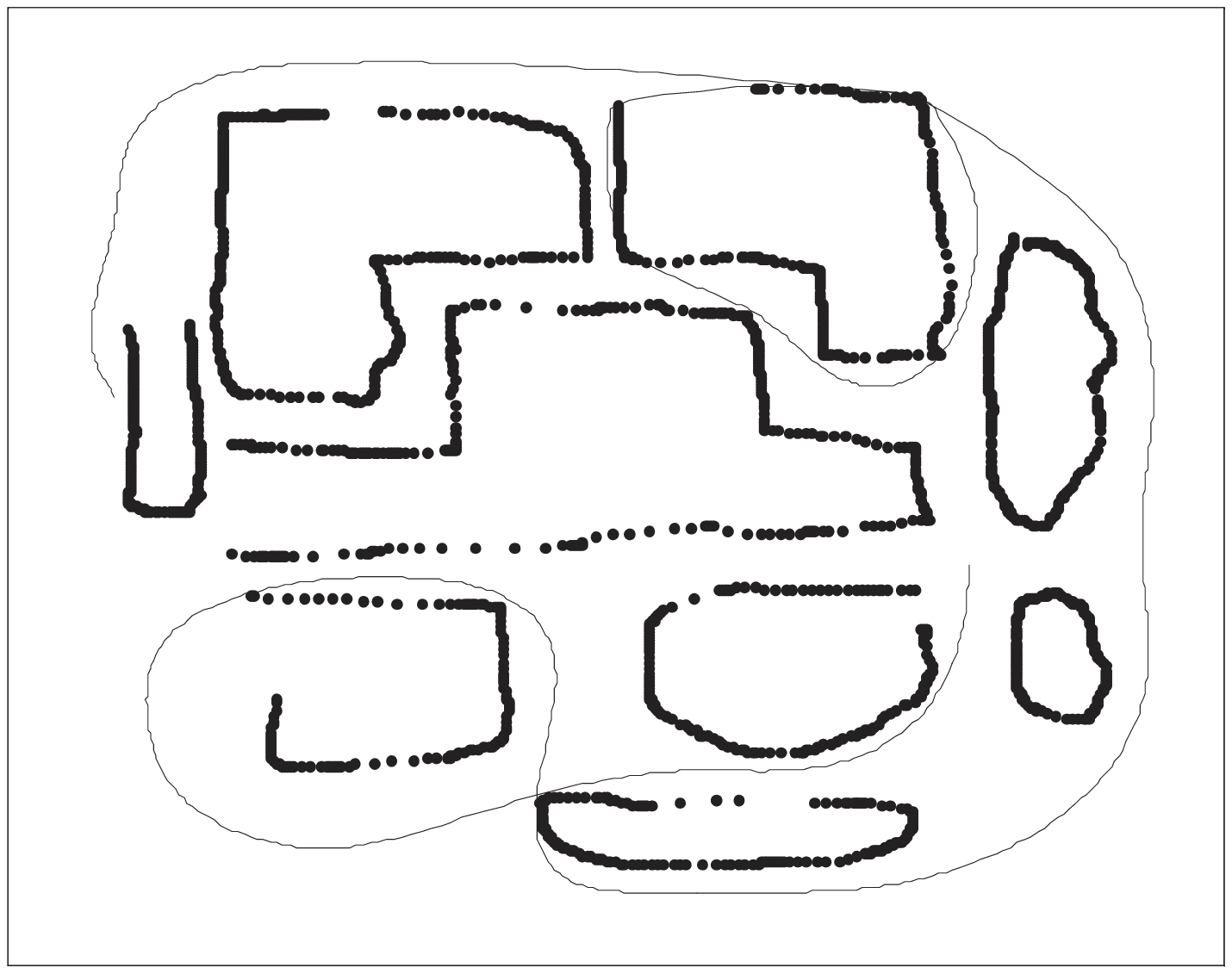}
    }
    \hfil
    \subfigure[After 20 sweeps.]
    {
        \label{Fig10d}
        \includegraphics[width=0.22\textwidth]{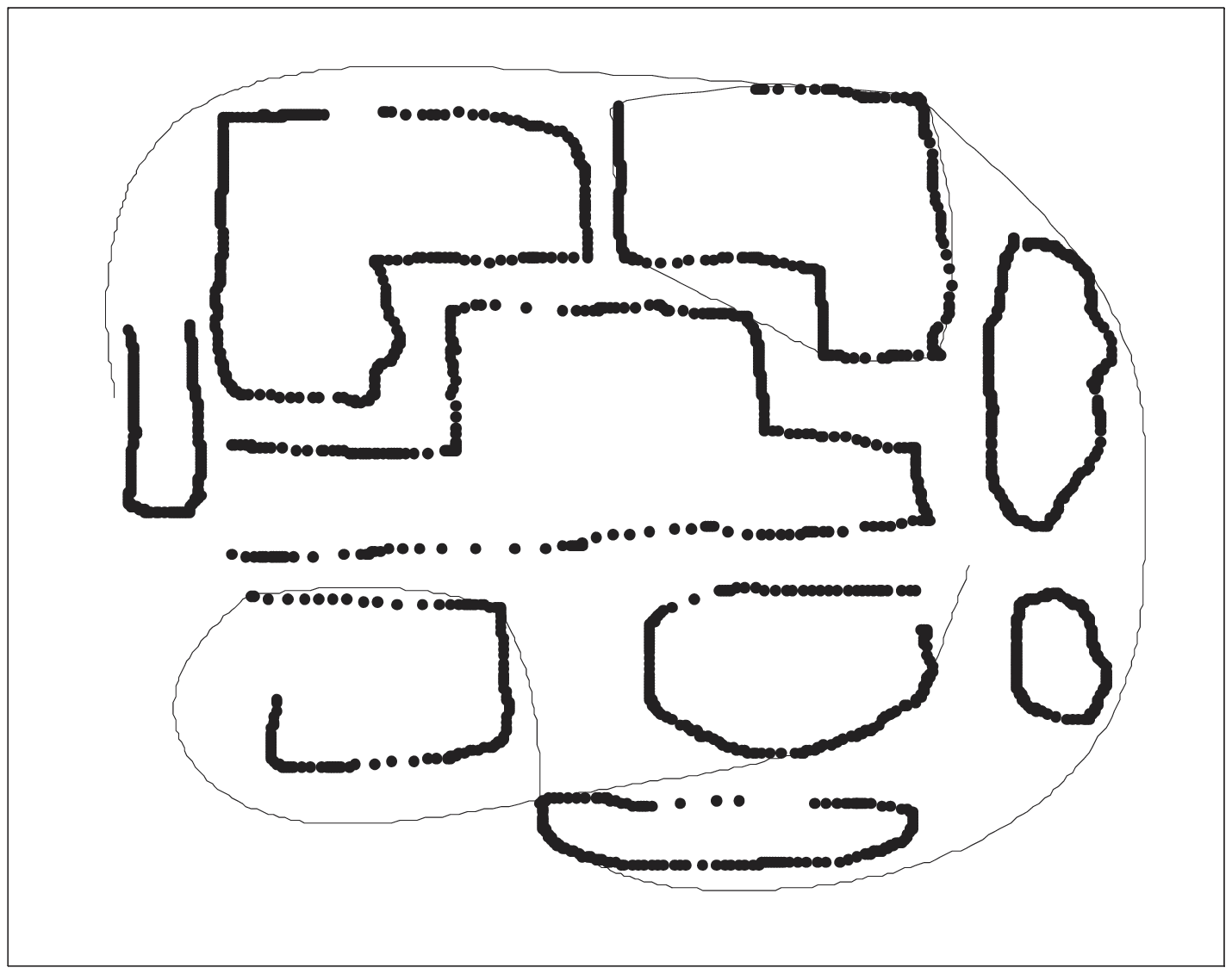}
    }
    \caption{STON finds a smooth homotopic path with respect to a
huge set (thousands) of obstacles. Dots represent point obstacles
while firm lines represent paths.}
\label{Fig10} 
\end{figure}

In real world applications, such as navigational planning of
mobile robots \cite{Quinlan1993} or route formation for military
planning \cite{Chandra2002,banerjee2003constructing,BanerjeeChandraPathASC2006,chandra2009diagrammatic,BanerjeeChandra2012}, the absolute shortest path is not
always desired; a suboptimal path that is devoid of sharp turns is
often more desirable in such cases. Fig. \ref{Fig10} shows the
capability of STON to produce such paths by appropriately
adjusting the parameter $\epsilon$. In this case, $\epsilon$ was
chosen to be 0.1\%. The illustration in Fig. \ref{Fig10} further
demonstrates the capability of STON to handle a huge number of
obstacles, in the range of a few thousand.

\begin{figure}
    \subfigure[At initialization.]
    {
        \label{Fig11a}
        \includegraphics[width=0.22\textwidth]{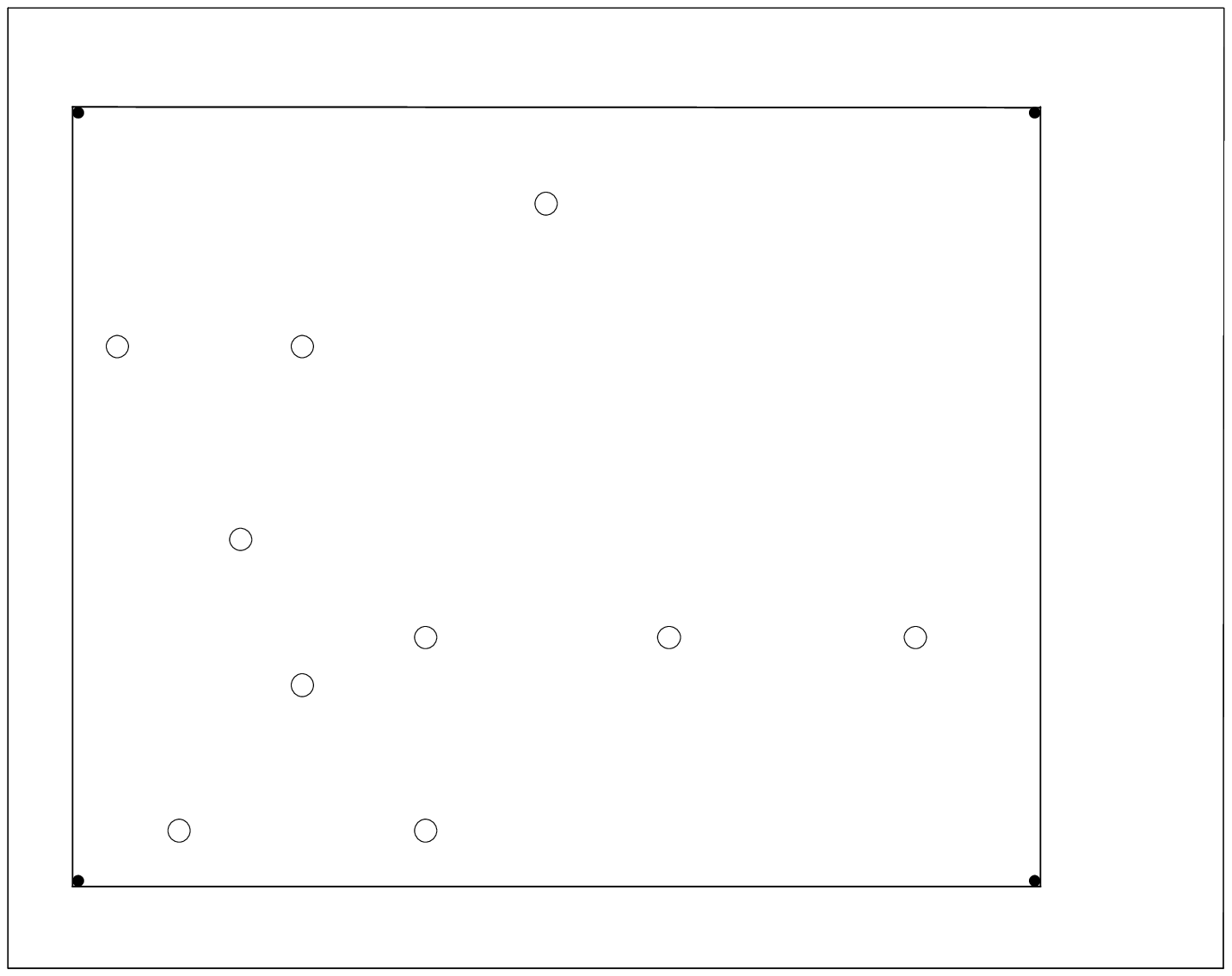}
    }
    \hfil
    \subfigure[After 10 sweeps.]
    {
        \label{Fig11b}
        \includegraphics[width=0.22\textwidth]{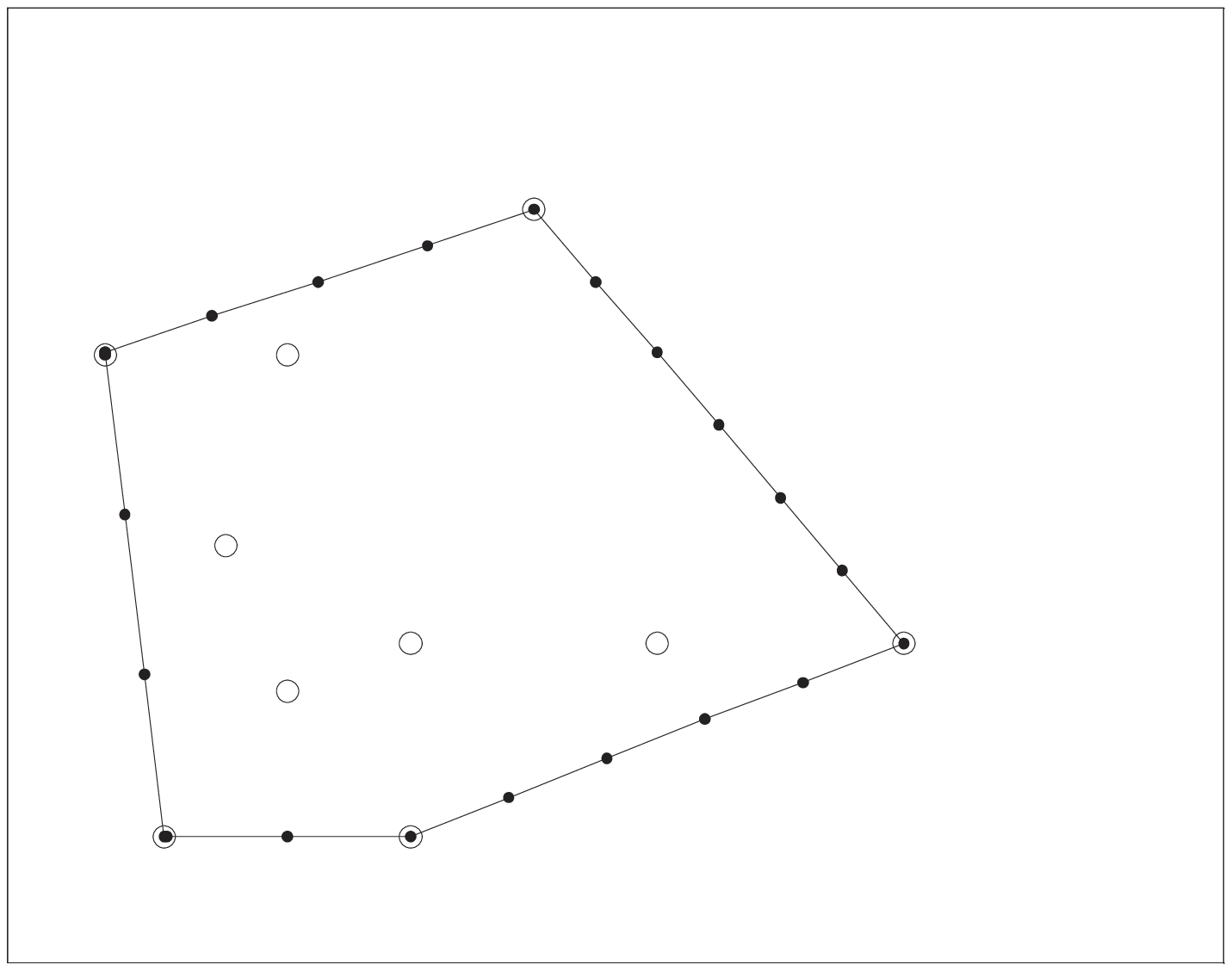}
    }
    \caption{STON computes the convex hull of a set of points
represented by circles. Lines represent the contour of the convex
hull while dots on the contour represent locations of the
processors.}
\label{Fig11} 
\end{figure}

STON can be used to compute the convex hull of a set of points, as
shown in Fig. \ref{Fig11}. A path has a starting and an ending
point which are fixed and common for all paths belonging to the
same homotopy class. To exploit this information for computation
of shortest homotopic paths, the first and the last processors,
$q_{1}$ and $q_{k}$, on a path with $k$ processors were assumed to
be fixed and their corresponding weights were never updated.
Computation of convex hull however does not require any fixed
processors, so weights corresponding to all processors were
updated. The starting and ending points were assumed to be the
same. Such a modification of STON makes it functionally similar to
an elastic band or a Snake \cite{Kassetal1988}.

Experiments with a number of different data sets, a few of which
are shown in Fig. \ref{Fig7} through \ref{Fig11}, assuming the
learning constant $\beta$ to be $10^{-2}$, reveal certain
characteristics of STON. It is expected that the total number of
sweeps required for convergence increases with the increase in
number of processors. Outcomes of our experiments satisfy such
expectations but they also show that average number of sweeps per
processor required for convergence decreases with the increase in
number of processors (see Fig. \ref{Fig13}). This is important in
determining how many processors to sample a path with as one
should choose the optimum number of processors for minimizing
computational costs. The errors in Fig. \ref{Fig13} refer to the
ratio of the length of the shortened path at convergence with
respect to the length of the shortest path.

\begin{figure}
\centering
        \includegraphics[width=0.45\textwidth]{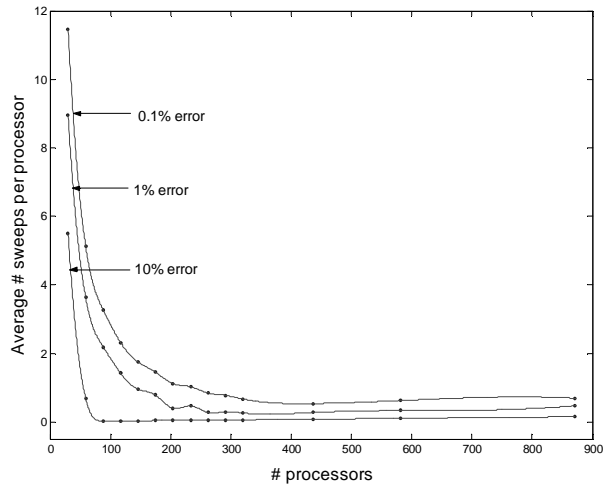}
    \caption{Experiments show that average number of sweeps per
processor required for convergence using STON decreases with the
increase in number of processors.}
        \label{Fig13}
\end{figure}

The learning rate is an important factor for ensuring faster
convergence. The number of processors being fixed, total number of
sweeps required for convergence increases with decrease in
learning constant $\beta$ (see Fig. \ref{Fig14}). We experimented
with a number of data sets including those shown in Fig.
\ref{Fig7} through \ref{Fig11}, sampling the paths with 30, 59,
88, 117, 146, 175 processors at different locations and varying
the learning constant $\beta$ from $10^{-4}$ to 0.2 for each data
set. Choosing a very high learning constant might lead to
suboptimal results as has been illustrated in Fig. \ref{Fig7}. It
is interesting to note from Fig. \ref{Fig14} that for low learning
constants, such as $10^{-2}$ or lesser, the total number of sweeps
required for convergence is more for lesser number of processors.
This observation only reinforces the fact that average number of
sweeps per processor decreases with the increase in number of
processors for low learning constants.

\begin{figure}
\centering
        \includegraphics[width=0.45\textwidth]{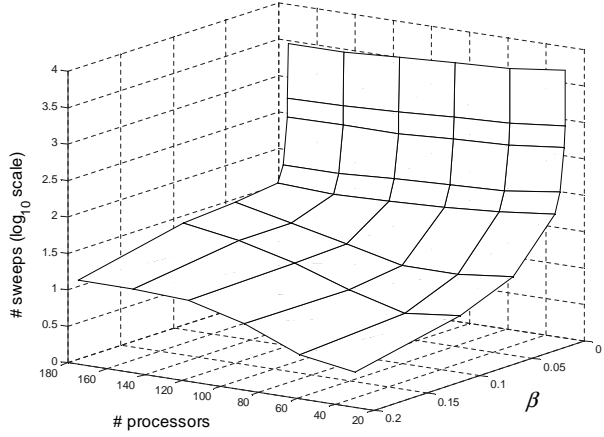}
    \caption{As the learning constant decreases, the number of
sweeps required for convergence using STON increases.}
        \label{Fig14}
\end{figure}

\section{Conclusion}
\label{Conclusion}

A self-organizing neural network algorithm STON is proposed that
models the phenomenon undergone by the particles forming a string
when the string is tightened from one or both of its ends amidst
obstacles. Discussions of the properties and correctness of this
anytime algorithm is presented assuming the given string is
sampled at a frequency of at least $d/2$ where $d$ is the minimum
distance between the obstacles. It is shown how STON might be
extended for tightening strings when the above constraint on
sampling is not met. This extension is applied to compute the
shortest homotopic path with respect to a set of obstacles. Proof
of correctness and computational complexity of the extension of STON are included. Experimental
results show that the proposed algorithm works correctly with both
simple and non-simple paths in reasonable time as long as the
constraints for correctness are met. STON is used to generate
smooth and shorter homotopic paths, a problem that can be modeled
as the phenomenon of tightening a string. STON is also used as an
elastic band for computing convex hulls. Future research aims at improving the computational complexity of the extension of STON and using it to solve more problems that can be mapped into the
problem of tightening a string or an elastic band.

\appendix
\section{Appendices}

\subsection{A Finite Sampling Theorem}
\label{A Finite Sampling Theorem}

A string in $\Re^{2}$, wound around point obstacles, can be
finitely sampled in such a way so as to guarantee only one
obstacle within the triangular area spanned by any three
consecutive points on the string. The following theorem states the
constraint necessary to be imposed on the sampling.

\begin{theorem}\label{Theorem:Finite sampling theorem} 
Sampling a string at half the minimum distance
between the obstacles guarantees at most one obstacle in any
triangle formed by three consecutive points on the string.
\end{theorem}

\begin{proof}
Let $d$ be the minimum distance between any two unique obstacles
in $P$, $P$ being the set of point obstacles. Let us sample the
string such that the distance between any two consecutive points
on the string is at most $d/2$. Since $d$ is finite, clearly this
leads to a finite sampling of the string. The theorem claims, this
sampling ensures that a triangle formed by any three consecutive
points on the string will never contain more than one unique
obstacle.

For contradiction, let us assume, there exists two obstacle points
in a triangle formed by three consecutive points $q_{i-1}$,
$q_{i}$, $q_{i+1}$ (see Fig. \ref{Fig15}). The segments
$\overline{q_{i-1}q_{i}}$ and $\overline{q_{i}q_{i+1}}$ included
in the string that form two sides of the triangle are each of
length at most $d/2$. Thus the maximum distance between any two
points lying within the triangle is less than $d$. But the
distance between any two obstacle points is at least $d$. Hence, a
contradiction, and the claim follows.
\end{proof}

\begin{figure}
\centering
        \includegraphics[width=0.45\textwidth]{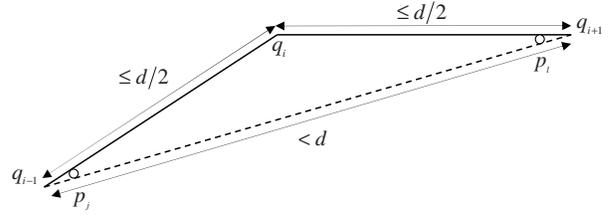}
    \caption{Sampling a path uniformly at half the minimum distance
between the obstacles guarantees at most one obstacle in any
triangle.}
        \label{Fig15}
\end{figure}

\subsection{Procedure for Introducing Processors in Convex Hull}
\label{Introduce processors in convex hull}

Here we describe the procedure for introducing processors near
each vertex of the convex hull in the extension of STON. The newly
introduced processor, say $q_{i}$, should be placed at a location
such that connecting it with the neighboring processors, $q_{i-1}$
and $q_{i+1}$, does not alter the homotopy class of the path i.e.
the path in which the new processors are being introduced should
remain in the same homotopy class as the given path. This is not a
trivial task as illustrated in Fig. \ref{Fig16}, where all the
processors are introduced near the convex hull but the new path
$\pi_{new}$ is not homotopic to the old path $\pi_{old}$.

\begin{figure}
\centering
        \includegraphics[width=0.4\textwidth]{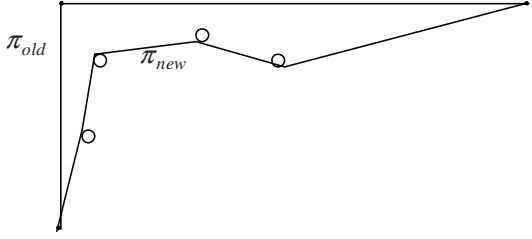}
    \caption{Introducing processors anywhere near the vertices
(shown by circles) of the convex hull does not guarantee the new
path will be homotopic to the old one.}
        \label{Fig16}
\end{figure}

In order for the new path to remain in the same homotopy class as
the old one, processors cannot be introduced within the convex
hull and lines joining consecutive processors cannot intersect the
edges of the convex hull. 

\begin{theorem}\label{Theorem:Introducing processors in convex hull}
If processors are introduced outside the convex hull in the regions bounded by the extended adjacent edges of the convex hull, then the lines joining the consecutive processors will not intersect the edges of the convex hull.
\end{theorem}

\begin{proof}
Let $V_{1}V_{2}V_{3}...V_{m}$ be a $m$-sided polygon which is the
convex hull for a set of obstacles under consideration (see Fig.
\ref{Fig17}). Let $q_{i}$ be a processor in the region formed by
extensions of adjacent edges $\overline{V_{i-1}V_{i}}$ and
$\overline{V_{i+1}V_{i}}$, $\forall i$, $1\leq i\leq m$. The claim
states that there cannot be an intersection between the line
segment $\overline{q_{j}q_{j+1}}$ and any edge of the convex hull.

\begin{figure}
\centering
        \includegraphics[width=0.45\textwidth]{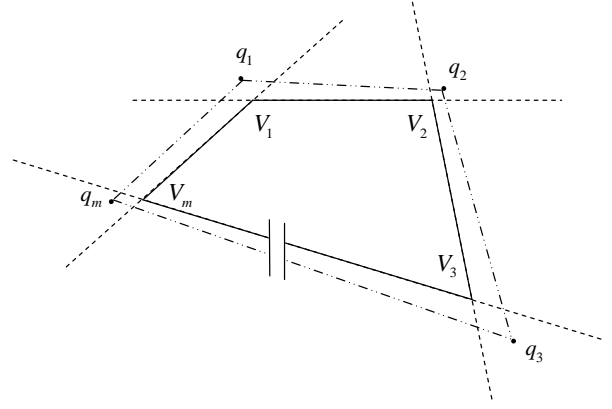}
    \caption{Introduction of processors outside the convex hull in
the regions bounded by the extended adjacent edges of the convex
hull guarantees that the lines joining the consecutive processors
will not intersect the edges of the convex hull.}
        \label{Fig17}
\end{figure}

Let us assume, for a contradiction, there exists at least one
intersection between the line segment $\overline{q_{j}q_{j+1}}$
and an edge, say $\overline{V_{k}V_{k+1}}$, of the convex polygon
$V_{1}V_{2}V_{3}...V_{m}$. Then, $q_{j}$ and $q_{j+1}$ must lie on
the opposite sides of the extended line segment
$\overline{V_{j}V_{j+1}}$. But by construction, the processors
$q_{j}$, $q_{j+1}$ lie on the same side of the extended line
segment $\overline{V_{j}V_{j+1}}$. Hence a contradiction and the
claim follows.
\end{proof}

\bibliographystyle{unsrt}
\bibliography{/Bonny/Research/LatexDocs/mybibfile}

%


\end{document}